%% file: main.tex
\documentclass[Afour,sageh,times]{sagej}

\usepackage{ulem}

\input{header.tex}

\newcommand\BibTeX{{\rmfamily B\kern-.05em \textsc{i\kern-.025em b}\kern-.08em
T\kern-.1667em\lower.7ex\hbox{E}\kern-.125emX}}

\def\volumeyear{2021}

\begin{document}

\runninghead{Zherong and Min}
\title{Joint Search of Optimal Topology and Trajectory for Planar Linkages}
\author{Zherong Pan\affilnum{1, 4, *} and Min Liu\affilnum{2, *} and Xifeng Gao\affilnum{3, 4} and Dinesh Manocha\affilnum{2}}

\affiliation{
\affilnum{*} indicates equal contribution.
\affilnum{1}Department of Computer Science, 
University of Illinois at Urbana-Champaign, Illinois IL 61801, USA.
\affilnum{2}Department of Computer Science and Electrical \& Computer Engineering, 
University of Maryland at College Park, Maryland MD 20742, USA.
\affilnum{3}Department of Computer Science,
Florida State University, Florida FL 32306, USA.
\affilnum{4}Lightspeed \& Quantum Studio, Tencent America, USA.
}
\allowdisplaybreaks

\corrauth{Min Liu}

\email{gfsliumin@gmail.com}

\begin{abstract}
We present an algorithm to compute planar linkage topology and geometry, given a user-specified end-effector trajectory. Planar linkage structures convert rotational or prismatic motions of a single actuator into an arbitrarily complex periodic motion, \refined{which is an important component when building low-cost, modular robots, mechanical toys, and foldable structures in our daily lives (chairs, bikes, and shelves). The design of such structures require trial and error even for experienced engineers. Our research provides semi-automatic methods for exploring novel designs given high-level specifications and constraints.} We formulate this problem as a non-smooth numerical optimization with quadratic objective functions and non-convex quadratic constraints involving mixed-integer decision variables (MIQCQP). We propose and compare three approximate algorithms to solve this problem: mixed-integer conic-programming (MICP), mixed-integer nonlinear programming (MINLP), and simulated annealing (SA). We evaluated these algorithms searching for planar linkages involving $10-14$ rigid links. Our results show that the best performance can be achieved by combining MICP and MINLP, leading to a hybrid algorithm capable of finding the planar linkages within a couple of hours on a desktop machine, which significantly outperforms the SA baseline in terms of optimality. We highlight the effectiveness of our optimized planar linkages by using them as legs of a walking robot.
\end{abstract} %reviewed


\maketitle

\input{introduction.tex} %reviewed
\input{related.tex} %reviewed
\input{problem.tex} %reviewed
\input{MIQCQP.tex} %reviewed
\input{solver.tex} %reviewed
\input{result.tex} %reviewed
\input{conclusion.tex} %reviewed

\section*{Acknowledgement}
This research is supported in part by ARO grant W911NF-18-1-0313, and Intel.
\bibliographystyle{SageH}
\bibliography{reference}
\end{document}

%% file: header.tex
\usepackage{amsmath,amsfonts,amssymb,amsthm}

\usepackage{prettyref}
\usepackage{mathrsfs}
\usepackage{graphicx}
\usepackage{wrapfig}
\usepackage{subfig}
\usepackage{MnSymbol}
\usepackage{multirow}
\usepackage{booktabs}
\usepackage{algorithm}
\usepackage[table]{xcolor}
\usepackage[noend]{algpseudocode}
\usepackage{mathtools}
\usepackage{enumitem}
\usepackage{comment}
\usepackage[colorlinks,allcolors=gray,hypertexnames=true]{hyperref} 

\usepackage[nomessages]{fp}

\usepackage{pgf}
\pgfkeys{/pgf/number format/.cd ,precision=2,sci generic={exponent={\times 10^{#1}}}}

\newtheorem{lemma}{\TE{Lemma}}
\newtheorem{assume}{\TE{Assumption}}

\algnewcommand{\LineComment}[1]{\State \(\triangleright\) #1}
\algdef{SE}[DOWHILE]{Do}{doWhile}{\algorithmicdo}[1]{\algorithmicwhile\ #1}
\newcommand*{\colorboxed}{}
\def\colorboxed#1#{%
  \colorboxedAux{#1}%
}
\newcommand*{\colorboxedAux}[3]{%
  \begingroup
    \colorlet{cb@saved}{.}%
    \color#1{#2}%
    \boxed{%
      \color{cb@saved}%
      #3%
    }%
  \endgroup
}
\newrefformat{fig}{Figure~\ref{#1}}
\newrefformat{par}{Section~\ref{#1}}
\newrefformat{appen}{Appendix~\ref{#1}}
\newrefformat{sec}{Section~\ref{#1}}
\newrefformat{sub}{Section~\ref{#1}}
\newrefformat{table}{Table~\ref{#1}}
\newrefformat{ass}{Assumption~\ref{#1}}
\newrefformat{alg}{Algorithm~\ref{#1}}
\newrefformat{def}{Definition~\ref{#1}}
\newrefformat{thm}{Theorem~\ref{#1}}
\newrefformat{lem}{Lemma~\ref{#1}}
\newrefformat{step}{Step~\ref{#1}}
\newrefformat{ln}{Line~\ref{#1}}
\newrefformat{eq}{Equation~\ref{#1}}
\newrefformat{pb}{Problem~\ref{#1}}
\newrefformat{it}{Item~\ref{#1}}
\newrefformat{te}{Term~\ref{#1}}
\def\Eqref Eq:#1:{\eqref{eq:#1}}
\newrefformat{Eq}{Equation~\Eqref#1:}
\newcommand{\prettyreft}[1]{\text{\prettyref{#1}}}

\newcommand{\E}[1]{\mathbf{#1}}
\newcommand{\TE}[1]{\textbf{#1}}

\newcommand{\FPP}[2]{\frac{\partial{#1}}{\partial{#2}}}
\newcommand{\FPPR}[2]{{\partial{#1}}/{\partial{#2}}}

\newcommand{\TWO}[2]{\left(\setlength{\arraycolsep}{1pt}\begin{array}{cc}{#1}, & {#2}\end{array}\right)}
\newcommand{\TWOC}[2]{\left(\setlength{\arraycolsep}{1pt}\begin{array}{c}#1 \\ #2\end{array}\right)}

\newcommand{\FIVE}[5]{\left(\setlength{\arraycolsep}{1pt}\begin{array}{ccccc}{#1}, & {#2}, & {#3}, & {#4}, & {#5}\end{array}\right)}

\newcommand{\MTT}[4]{\left(\setlength{\arraycolsep}{1pt}\begin{array}{cc}#1 & #2 \\ #3 & #4\end{array}\right)}

\newcommand{\argcon}{\text{s.t.}\quad}
\newcommand{\argmin}[1]{\underset{#1}{\text{argmin}}\;}

\newcommand{\NN}{\E{n}}

\newcommand{\BO}{\mathcal{O}}
\newcommand{\SOSA}{\mathcal{SOS}_1}
\newcommand{\SOSB}{\mathcal{SOS}_2}

\newcommand{\CND}[2]{\NN_{{#1\rightarrow #2}}}
\newcommand{\CNN}[3]{\NN_{{#1\rightarrow #2#3}}}
\definecolor{motor}{HTML}{289942}
\definecolor{fixed}{HTML}{D02824}
\definecolor{movable}{HTML}{000000}
\definecolor{endeffector}{HTML}{3AB4E5}
\newcommand{\GAMMA}[5]{
\FPeval{\X}{cos((#1*45+45/2)*pi/180)*#4-#2}
\FPeval{\Y}{sin((#1*45+45/2)*pi/180)*#4+#3}
\put(\X,\Y){#5}
}
\newcommand{\GAMMAI}[5]{
\FPeval{\X}{cos((#1*45)*pi/180)*#4/2-#2}
\FPeval{\Y}{sin((#1*45)*pi/180)*#4/2+#3}
\put(\X,\Y){#5}
}

\usepackage{xcolor}
\definecolor{Blue}{rgb}{0.2, 0.2, 0.8}
\newcommand{\refined}[1]{#1}

%% file: introduction.tex
\section{Introduction}
Over the past decades, robots have profoundly changed the industry, scaling up the productivity and simplifying the workflow of assembly lines. With their superior reliability and accuracy, however, comes a high fabrication cost and time-consuming maintenance. For robots to serve other aspects of our lives, they have to be versatile and adapt to rapidly changing tasks. To this end, we have witnessed an ongoing trend in both industry \citep{hebi_robotics} and research \citep{1159221,9115061} communities towards low-cost, modular robot designs. This type of modular hardware opens the door to a huge space of infinitely many robot designs that can be exploited to accomplish a variety of tasks. A pivotal challenge faced by a robot designer is to determine the ``optimal'' design to accomplish a given task. 

\begin{figure}[ht]
\centering
\includegraphics[width=0.45\textwidth]{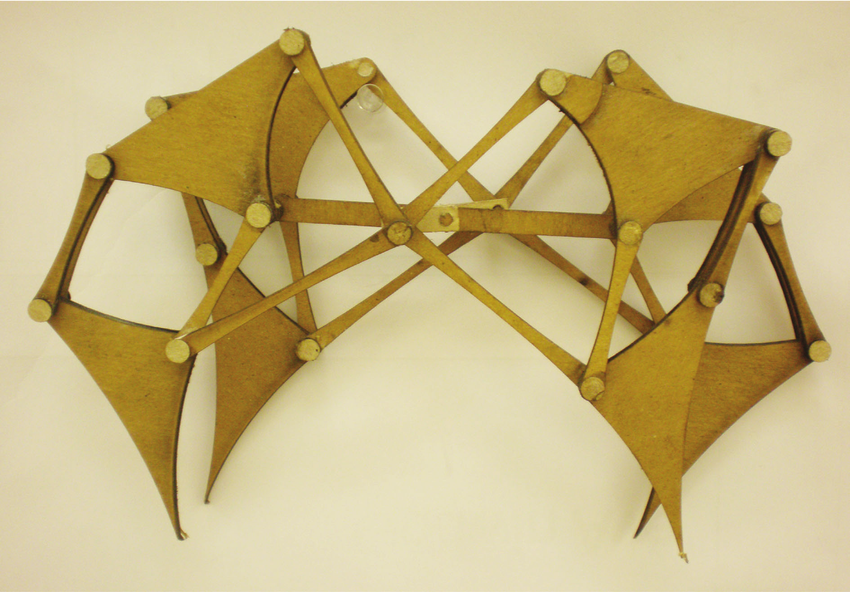}
\caption{\label{fig:strandbeest} The Jansen's strandbeest can be manufactured by putting together 4 planar linkages to the left and right of the robot's body; see \cite{nansai2013dynamic} for more details.}
\end{figure}
This paper addresses the problem of computational task-driven, design optimization for planar linkages. Planar linkages are mechanical structures built out of a set of rigid bodies connected by hinge joints. These structures are capable of converting simple rotational or linear motion into arbitrarily complex, curved motion. Being electronics-free and low-cost, planar linkages have been studied for centuries and used ubiquitously in mechanical tools, household accessories, and vehicles. As illustrated in \prettyref{fig:strandbeest}, they can also be integrated into low-cost robots to fulfill requirements of different types of locomotion for walking, swimming, and flying \citep{HERNANDEZ2016AHO,Thomaszewski:2014:CDL:2601097.2601143}. However, the design of planar structures requires significant trial and error, even for experienced engineers.

Computational robot design \citep{Ha2017JointOO,ha2018computational,whitman2020modular} has recently drawn increasing attention, partly due to high-performance computers. The decision space can grow exponentially with the complexity of the robot, so these methods are limited to designing sequential manipulators with no more than a few links. These links are standardized with fixed geometric shapes, so the design algorithm only needs to determine the connectivity (or topology) of those links, leaving only a discrete decision space of design variables to the designer. Prior works utilize algorithms such as $A^*$ search \citep{ha2018computational} and Q-learning \citep{whitman2020modular} to solve the underlying optimization problem. Standardized links are sufficient for sequential manipulators because they can use multiple actuators to move the end-effector. With a single actuator, however, the end-effector of a planar linkage can only trace out a single curve, and a design algorithm must jointly search for both the geometry and topology of links to cover a large variety of curve shapes, which poses a much more challenging decision-making problem in a mixed continuous-discrete space.

\TE{Main Results:} We propose a new algorithm to automatically search for topology and geometry for a large sub-class of planar linkages that are specified by prior works \cite{kecskemethy1997symbolic,Thomaszewski:2014:CDL:2601097.2601143,bacher2015linkedit}. Our main idea is to formulate the problem as a mixed-integer numerical optimization with quadratic objectives and non-convex, quadratic constraints, or MIQCQP. We then propose and compare the performance of three algorithms to approximate the optimal solutions: a mixed-integer conic programming (MICP) algorithm that uses a piecewise convex relaxation of non-convex constraints; a mixed-integer nonlinear programming (MINLP) algorithm that uses sequential quadratic programming (SQP) to find locally feasible solutions for non-convex constraints; a simulated annealing (SA) algorithm that randomizes both the geometry and topology, which is a variant of \cite{Zhu:2012:MMT:2366145.2366146}. We have evaluated our method in a row of optimization tasks with $5-7$ rigid bodies tracing out complex end-effector curves. These hybrid algorithms can find a solution within a couple of hours on a desktop machine, and the results exhibit an averaged $9.3\times$ higher optimality when compared with the SA-baseline. 

\refined{This paper is an extended version of our prior work \cite{ourwork2019ISRR}, where we proposed the original MICP relaxation scheme to solve MIQCQP approximately. We extend over the prior work in three ways. First, we propose a new MIQCQP-approximation scheme based on MINLP. MINLP solver tries to satisfy non-convex constraints exactly and achieves a better balance between computational time and the optimality of the resulting solution.. Second, we initialize MINLP solver using a similar approach as MICP, which allows MINLP to exhaustively try more initial guesses and improves its success rate. We also introduce a local optimization move into the SA baseline algorithm to improve its efficacy. Finally, we conduct simulated experiments to illustrate the application of our optimized planar linkages on robot locomotion. Specifically, we use a linkage as legs of a walking robot and optimize its dynamics properties to maximize the walking performance via Bayesian exploration.}

%% file: related.tex
\section{\label{sec:related}Related Work}
In this section, we review related work in robot design optimization, mixed-integer programming, and planar-linkage design.

\TE{Robot Design Optimization} is among the most challenging decision-making problems because the design algorithm must jointly reason about the robot design parameters and motion plans. This problem is a superset of conventional topology and truss optimization \citep{LIU2016161}, which does not involve movable components. Furthermore, the decision space of a robot design is oftentimes high-dimensional, and involves topology, geometry, and space-time variables. Existing approaches use one or more of those three variables to approximately search for optimal robot designs. In \cite{umetani2014pteromys,Thomaszewski:2014:CDL:2601097.2601143,bacher2015linkedit}, authors proposed human-in-the-loop design tools that either visualize the designed robot motion or locally optimize the robot's continuous geometric parameters. Our method is complementary to these works as it jointly optimizes the topology or geometry, although we can still find sub-optimal solutions or even fail to find a solution. In \cite{Zhu:2012:MMT:2366145.2366146,8793802}, the authors use stochastic optimization solvers such as SA and Bayesian optimization to search for robot topology. Our method provides an alternative, deterministic approach to solve the same problem. Most recent works \citep{Ha2017JointOO,spielberg2017functional,saar2018model,8794333} locally optimize robot's geometry given a user-provided topology and a geometric initial guess. However, figuring these topology and geometric initial guesses can still be labor-intensive.

\TE{Mixed-Integer Programming (MIP)} is a standardized tool to formulate mathematical programming problems with non-convex constraints that can be expressed as a disjoint set. Although solving general MIP is NP-hard, practical branch-and-bound (BB) algorithms \citep{lawler1966branch} can find global optima of MICP instances of small-to-medium sizes, where each member of the disjoint set is convex. BB algorithms rely on tight, convex relaxations to efficiently find lower bounds and cut off sub-optimal solutions at an early stage. BB serves as the computational engine of a large variety of problems, including inverse kinematics \citep{dai2017global}, network flows \citep{conforti2009network}, mesh generations \citep{bommes2009mixed}, motion planning with collision handling \citep{ding2011mixed}, and legged locomotion \citep{deits2014footstep}. We adopt a similar technique as these methods to formulate our topology optimization problem, where constraints with integer variables ensure the correctness of link connectivity. However, if the members of the disjoint set are non-convex, as it is the case with our geometric optimization problem, finding the exact global optima is intractable and two approaches can be used to approximate them. First, big-M methods \citep{bertsimas1997introduction}, McCormick envelopes, and piecewise approximations \citep{liberti2004reformulation} discretize a non-convex set as a union of convex sets, where the discretization error can be made arbitrarily small using higher resolutions of discretization and more integer decision variables. Second, MISQP algorithms \citep{exler2007trust} locally solve non-convex programs and use the solution in BB algorithms, which is not guaranteed to be a lower bound. Prior works like \cite{lobato2003mixed,kanno2013topology} have formulated topology optimization problems as MIP. Nevertheless, our work is the first one to formulate the planar linkage problem as MIP, and to employ MIP to find the optimal topology, geometry, and trajectory of a planar linkage concurrently.

\TE{Planar Linkage} is a set of 2D rigid objects connected together to convert motion and forces. We considered a subset of planar linkages connected by hinge joints. More generally, linkage structures can be coupled with pistons, gears, springs \citep{Zhu:2012:MMT:2366145.2366146}, and compliant structures \citep{megaro2017computational} to exhibit more complex motion of the end-effector. But their designs are still relying heavily on human experiences. Parallel to their applications in robotics, the mathematical structure of planar linkages has been studied for centuries. In 1875, \cite{kempe1875general} provided a constructive method to build a planar linkage that can trace out any algebraic curve, but the resulting linkage structure can be extremely complex. The recent work \cite{gallet2017planar} proposed a construction leading to simpler structures, but their complexities are still too high for real-world applications. Therefore, practitioners rely on genetic algorithms \citep{Zhu:2012:MMT:2366145.2366146,cabrera2002optimal} or semi-automatic design tools \citep{Thomaszewski:2014:CDL:2601097.2601143,bacher2015linkedit} to search for simple linkage structures that trace out a specified end-effector curve.

%% file: problem.tex
\section{\label{sec:problem}Planar Linkages Optimization Problem}
\begin{figure*}[ht]
\centering
%\vspace{-15px}
\begin{tabular}{@{}c@{}}
\includegraphics[width=0.25\textwidth]{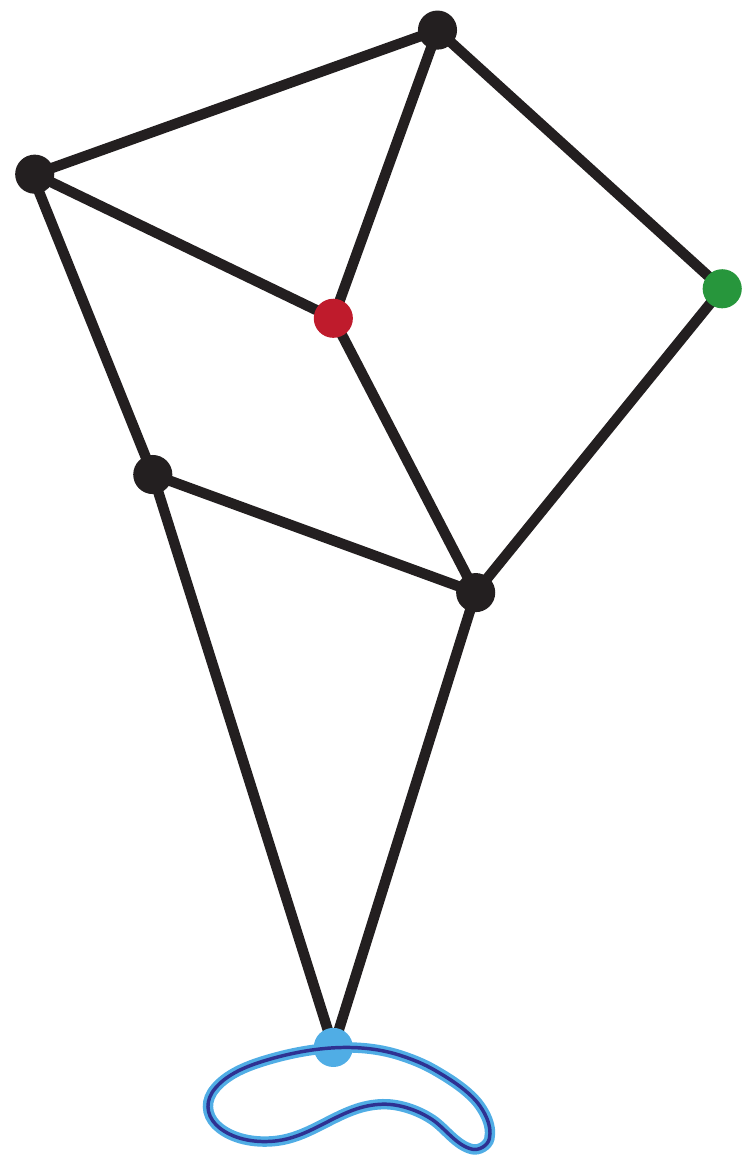}
\put(  2 ,140){\textcolor{motor}{\tiny$\NN_1,U_1=1$}}
\put(-62 ,137){\textcolor{fixed}{\tiny$\NN_2,U_2=1$}}
\put(-44 ,187){\textcolor{movable}{\tiny$\NN_3,U_3=1$}}
\put(-132,171){\textcolor{movable}{\tiny$\NN_4,U_4=1$}}
\put(-39 ,93 ){\textcolor{movable}{\tiny$\NN_5,U_5=1$}}
\put(-93 ,116){\textcolor{movable}{\tiny$\NN_6,U_6=1$}}
\put(-62 ,23 ){\textcolor{endeffector}{\tiny$\NN_7,U_7=1$}}
\put(-26 ,168){\textcolor{movable}{\tiny$Q_{13}=0.5$}}
\put(-22 ,116){\textcolor{movable}{\tiny$Q_{15}=0.5$}}
\put(-59 ,155){\textcolor{movable}{\tiny$Q_{23}=0.5$}}
\put(-94 ,155){\textcolor{movable}{\tiny$Q_{24}=0$}}
\put(-57 ,120){\textcolor{movable}{\tiny$Q_{25}=0.5$}}
\put(-94 ,182){\textcolor{movable}{\tiny$Q_{34}=2$}}
\put(-132,130){\textcolor{movable}{\tiny$Q_{46}=3$}}
\put(-83 ,94 ){\textcolor{movable}{\tiny$Q_{56}=1$}}
\put(-53 ,58 ){\textcolor{movable}{\tiny$Q_{57}=1$}}
\put(-112,70 ){\textcolor{movable}{\tiny$Q_{67}=5$}}
\put(-100,40){(a)}
\end{tabular}
\hfill
\begin{tabular}{@{}c@{}}
\includegraphics[width=0.4\textwidth]{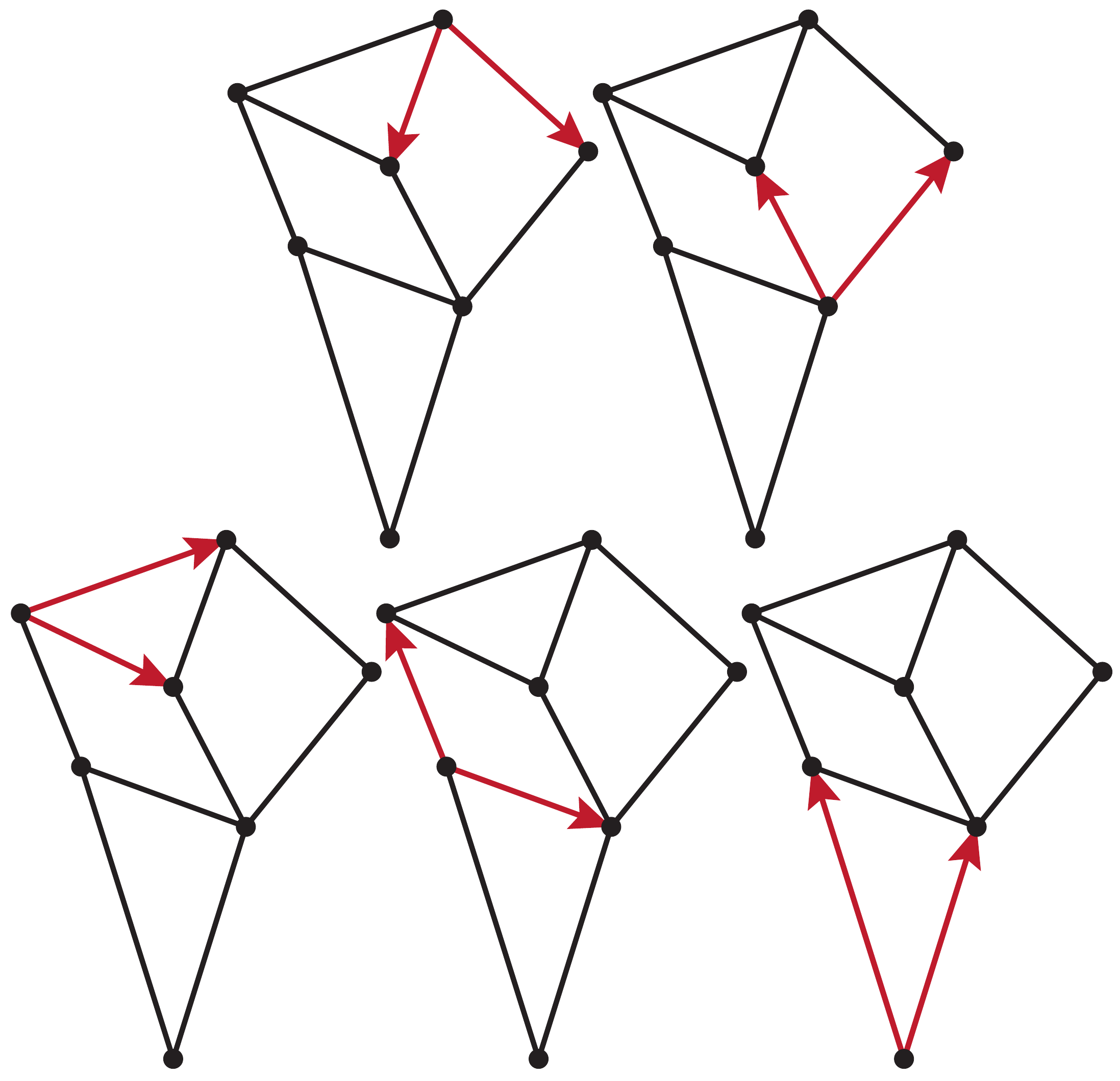}
\put(-120,115){$\CNN{3}{2}{1}$}
\put(-55 ,115){$\CNN{5}{2}{1}$}
\put(-157,25 ){$\CNN{4}{3}{2}$}
\put(-92 ,25 ){$\CNN{6}{5}{4}$}
\put(-27 ,25 ){$\CNN{7}{6}{5}$}
\put(-170,110){(b)}
\end{tabular}
\hfill
\begin{tabular}{@{}c@{}}
\includegraphics[width=0.2\textwidth]{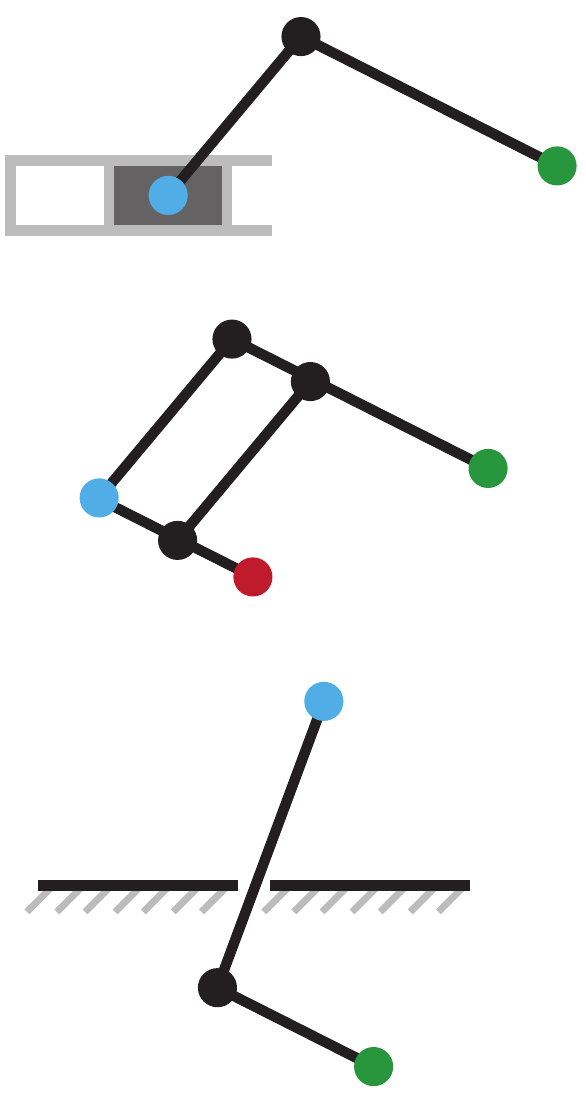}
\put(-80,60){(c)}
\end{tabular}
%\vspace{-10px}
\caption{\label{fig:linkage} (a): Jansen's mechanics is a planar linkage involving 7 nodes. The motor node $\NN_1$ is green, the fixed node $\NN_2$ is red, the movable nodes $\NN_{3,4,5,6}$ are black, and the end-effector node $\NN_7$ is blue. (b): Our method is based on a prior symbolic representation \citep{kecskemethy1997symbolic,bacher2015linkedit}. This representation assumes that each node is connected to zero or two other nodes with lower indices \refined{(marked by arrows)}: $\CNN{3}{2}{1}$, $\CNN{5}{2}{1}$, $\CNN{4}{3}{2}$, $\CNN{6}{5}{4}$, $\CNN{7}{6}{5}$. \refined{(c): We show three important linkage structures beyond our topology subset, which involve pistons, complex loops, and pinhole constraints.}}
%\vspace{-15px}
\end{figure*}
In this section, we define a subset of the planar linkages considered in this paper. As illustrated in \prettyref{fig:linkage} (a), we have a set of rod-like rigid bodies connected using hinge joints. The end points of these rigid bodies can take at most $N$ distinct positions, denoted as a node set $\NN_{1,\cdots,N}$, of which $\NN_1$ is the rotational motor and $\NN_N$ is the end-effector. Within one limit cycle, $\NN_1$ traces out a circular curve centered at $\TWO{X_C}{Y_C}$ with a radius $r$:
\begin{align}
\label{eq:motor}
\NN_1(t)=\TWO{\sin(\pm t)r+X_C}{\cos(\pm t)r+Y_C},
\end{align}
which induces trajectories of other nodes $\NN_i(t)$ via forward kinematics, where $t\in[0,2\pi)$ is the time parameter. Throughout the paper, we use $X,Y$ to denote the two axes of a 2D vector. The other $N-2$ nodes can be one of two kinds: fixed or movable. In addition, a rigid body may exist between each pair of nodes $\NN_{i,j}$, in which case $\|\NN_i(t)-\NN_j(t)\|$ must be a constant for all $t$. Given these definitions, we formulate the optimal planar linkage design problem as follows, where we take the following inputs:
\begin{itemize}[leftmargin=*]
\item A user-provided target end-effector trajectory $\NN_N^*(t)$. 
\item $K$: The maximal number of nodes in the planar linkage.
\item $T$: The number of samples needed to discretize the end-effector trajectory $\NN_N(t)$.
\item $S$: The parameter controlling the accuracy of the MICP formulation. A larger $S$ leads to greater accuracy and higher computational cost.
\end{itemize}
The output of our method is the tuple $\mathcal{L}=\left<N,C_{ji},\NN_i(t),X_C,Y_C,r\right>$ defining both the topology and geometry of a planar linkage:
\begin{itemize}[leftmargin=*]
\item An integer vector of size $N$ (the number of nodes) containing the type of each node: fixed or movable.
\item An $N\times N$ adjacent matrix $C^{N\times N}$ where $C_{ji}=1$ means a rigid body connects $\NN_i$ and $\NN_j$.
\item The position of $\NN_{1,\cdots,N}(t)$ at a certain, arbitrary time instance $t$.
\item $X_C,Y_C,r$ are determined automatically by our MICP formulation.
\end{itemize}
The goal of our method is to find the above set of variables that minimizes the cost $\int_0^{2\pi} \|\NN_N(t)-\NN_N^*(t)\|^2 dt$. For two planar linkage structures, we claim that one is more accurate or optimal than the other if its end-effector trajectory incurs a smaller cost.

\subsection{\refined{Constrained Linkage Kinematics}}
\refined{A valid planar linkage has only one degree of freedom, and the positions of all the nodes must be uniquely determined at a given time $t$ and a fixed initial configuration.} Therefore, planar linkages must use closed loops to eliminate all the redundant degrees of freedom. \refined{However, computing the forward kinematics for general, closed-loop articulated bodies involves solving constrained systems of equations \citep{featherstone2014rigid}, which increases the complexity of the search of their topological structures. Therefore, we limit our research to a subset of linkage topology, which was originally proposed by \cite{kecskemethy1997symbolic} and later adapted to human-assisted linked design in \cite{bacher2015linkedit}. The kinematics of this subset can be computed as easily as open-loop articulated bodies.}

\refined{The key to our kinematic computation lies in the law-of-cosine. Specifically, if a node $\NN_i$ is connected to two other nodes $\NN_j,\NN_k$  with known positions via rigid links with length $l_{ji},l_{ki}$, then the position of $\NN_i$ can be determined using the following function:
\begin{equation}
\begin{aligned}
\label{eq:forward}
&\NN_i(\NN_j,\NN_k,l_{ji},l_{ki})\triangleq
\frac{l_{ji}}{\|\NN_j-\NN_k\|}\E{R}(\NN_k-\NN_j)+\NN_j\\
&\E{R}\triangleq\MTT{\cos}{\pm\sqrt{1-\cos^2}}{\mp\sqrt{1-\cos^2}}{\cos}\\
&\cos\triangleq\frac{\|\NN_j-\NN_k\|^2+l_{ji}^2-l_{ki}^2}{2\|\NN_j-\NN_k\|l_{ji}},
\end{aligned}
\end{equation}
where $l_{ji}$ is the length of the link connecting $\NN_i$ and $\NN_j$ and $\E{R}$ is a rotation matrix. Note that, if we have $\|\NN_j-\NN_k\|>l_{ji}+l_{ki}$ or $\|\NN_j-\NN_k\|<|l_{ji}-l_{ki}|$ at certain time instance, then the three nodes cannot form a triangle and such a configuration cannot be realized. The above law-of-cosine has two solutions (determined by the sign of off-diagonal entries of $\E{R}$) corresponding to two mirrored triangles, but the linkage can only exhibit a unique, continuous motion without flipping any triangles. As proposed in \cite{kecskemethy1997symbolic}, we construct a planar linkage by recursively connecting a new node with two other nodes using rigid links. As a result, the position of each new node can be determined by the law-of-cosine. As illustrated in \prettyref{fig:linkage} (b), such topological constraints can be formalized using a connectivity graph $\mathcal{G}=\left<\mathcal{V},\mathcal{E}\right>$, where the vertex set $\mathcal{V}=\{\NN_1,\cdots,\NN_N\}$ consists of all the nodes and the edge set $\mathcal{E}=\{\CND{a}{b}\}$ consists of node pairs connected by a link, on which our topological constraint can be summarized below:
\begin{assume}
\label{ass:1}
The vertex set $\mathcal{V}$ of the connectivity graph has a topological ordering, by which each node $\NN_i$ is connected to either zero or two other nodes with lower indices (denoted as $\CNN{i}{j}{k}\triangleq\{\CND{i}{j},\CND{i}{k}\}$ with $i>j$ and $i>k$).
\end{assume}
With the topological ordering, we can determine the position of each node in ascending order. A node connected to zero lower-index nodes is either the actuator $\NN_1$ or a fixed node, otherwise, we apply the law-of-cosine to $\CNN{i}{j}{k}$ and determine the position of $\NN_i$. Using this method, the forward kinematics of the entire linkage structure can be computed within $\mathcal{O}(N)$ and Jacobian matrix with respect to link lengths can also be computed within $\mathcal{O}(N)$ using the adjoint method as summarized in \prettyref{alg:forward}. This Jacobian matrix can be used to locally optimize the geometry of a linkage structure using a gradient-based algorithm as done in \cite{bacher2015linkedit}. We emphasize that many important linkage structures, involving pistons, complex loops, or pinhole constraints, are beyond our subset, as illustrated in \prettyref{fig:linkage} (c). However, our subset already encompasses a rich variety of end-effector trajectories as illustrated by our results.}
\begin{algorithm}[ht]
\caption{Forward/Inverse Kinematics}
\label{alg:forward}
\begin{algorithmic}[1]
\State Compute $\NN_1$ using \prettyref{eq:motor}
\For{$i=2,\cdots,N$}\Comment{Forward Kinematics}
\If{$\NN_i$ is not fixed with $\CNN{i}{j}{k}\subset\mathcal{E}$}
\State Compute $\NN_i$ using \prettyref{eq:forward}
\State Output $\NN_i$
\EndIf
\EndFor
\For{$i=2,\cdots,N$}\Comment{Inverse Kinematics}
\If{$\NN_i$ is not fixed with $\CNN{i}{j}{k}\subset\mathcal{E}$}
\State Compute $\FPP{\NN_i}{l_{ji}},\FPP{\NN_i}{l_{ki}},\FPP{\NN_i}{\NN_j},\FPP{\NN_i}{\NN_k}$ using \prettyref{eq:forward}
\State $\FPP{\NN_N}{l_{ji}}\gets\FPP{\NN_N}{\NN_i}\FPP{\NN_i}{l_{ji}}\quad
\FPP{\NN_N}{l_{ki}}\gets\FPP{\NN_N}{\NN_i}\FPP{\NN_i}{l_{ki}}$
\State Output $\FPP{\NN_N}{l_{ji}},\FPP{\NN_N}{l_{ki}}$
\State $\FPP{\NN_N}{\NN_j}\gets\FPP{\NN_N}{\NN_i}\FPP{\NN_i}{\NN_j}\quad
\FPP{\NN_N}{\NN_k}\gets\FPP{\NN_N}{\NN_i}\FPP{\NN_i}{\NN_k}$\Comment{Adjoint}
\EndIf
\EndFor
\end{algorithmic}
\end{algorithm}

%% file: MIQCQP.tex
\section{Planar Linkage Optimization as MIQCQP}
In this section, we show that the optimal planar linkage design problem can be reformulated as a MIQCQP. Such reformulation allows us to utilize mature algorithms and approximation techniques to find (nearly) optimal solutions. Unlike prior methods \citep{Thomaszewski:2014:CDL:2601097.2601143,bacher2015linkedit} for planar linkage optimization that are based on minimal coordinates, we propose to use maximal coordinates to represent the configuration. Maximal coordinates treat all the node positions $\NN_i$ as independent decision variables and introduce additional constraints to ensure the rigidity of each link. Although maximal coordinates use more variables and involve solving constrained systems of equations, the constraints take a simpler form to be handled by numerical optimization tools. A similar idea has been used by \cite{dai2017global} to compute globally optimal inverse kinematics for sequential manipulators. Conceptually, our goal is to solve the following infinite-dimensional program:
\begin{equation}
\begin{aligned}
\label{eq:MIQCQP}
\argmin{C_{ji},\NN_i(t),X_C,Y_C}&\int_0^{2\pi}\|\NN_N(t)-\NN_N^*(t)\|^2 dt + \text{reg.}\\
\argcon{}&\text{Topological Constraints}\\
&\text{Geometric Constraints},
\end{aligned}
\end{equation}
where the main objective is to search for a linkage structure whose end-effector curve matches the user-provided target curve as much as possible. Meanwhile, we introduce a regularization term to reduce the fabrication cost, which penalizes the number of links and the total length of links. To ensure that the linkage is realizable and well-behaved, we introduce several sets of constraints as summarized in \prettyref{table:consTable}. The topological constraints ensure that the linkage structure satisfies \prettyref{ass:1}, and the geometric constraints ensure the kinematic feasibility. In the following sections, we use additional notations to mark the range of indices to which a constraint applies. If no notations are used, then the constraint applies to all the index combinations. Furthermore, we assume a variable is continuous with no bounds, unless otherwise specified (e.g., as a binary variable).
\begin{table}[ht]
\setlength{\tabcolsep}{1pt}
\begin{center}
\footnotesize
\begin{tabular}{ll}
\toprule
Constraint Set & Guarantees\\
\midrule
NodeUsageConstraint & unambiguous node type definitions\\
NodeConnectivityConstraint & node connectivity satisfies \prettyref{ass:1}\\
NoWasteConstraint & each node affects the end-effector trajectory\\
MovableNodeConstraint & movable nodes are connected to actuator\\
\midrule
RealizabilityConstraint & linkage structure can be fabricated\\
AreaConstraint & end-effector trajectory is unique\\ 
MotorConstraint & motor is rotational\\
\bottomrule
\end{tabular}
\end{center}
\caption{\label{table:consTable}\refined{A summary of topology and geometric constraint sets and the guarantees corresponding to each constraint.}}
\vspace{-20px}
\end{table}

\subsection{Topological Constraints}
\refined{We design four types of topological constraints. Our first set of constraints is denoted as \TE{NodeUsageConstraint}, which allows a numerical optimizer to automatically determine the number of nodes and links to use. We further ensure that each node can either be movable or fixed, but not both.} Since the number of nodes is unknown, we assume that the maximum number of nodes is $K\geq N$. We will have then all the joints defined, but only $N$ of them should be present, which will be formalized by introducing a binary indicator variable $U_i$. For each node other than the first motor node $\NN_1$, $U_i=1$ indicates that $\NN_i$ will be present as a part of the planar linkage structure. In addition, we need another binary variable $F_i$ such that $F_i=1$ indicates that $\NN_i$ is fixed and $F_i=0$ indicates that $\NN_i$ is movable. These two sets of variables are under the constraint that only a used node can be movable. In addition, we assume that the last node $\NN_K$ is the end-effector that must be used. In summary, we introduce the following sets of variables and node-state constraints:
\begin{equation}
\begin{aligned}
\label{eq:state}
&U_i,F_i\in\{0,1\}\\% \quad\forall i=1,\cdots,K \\
&1-F_i\leq U_i  \quad
U_1=U_K=1   \quad
F_1=0.
\end{aligned}
\end{equation}

\refined{Our next set of constraints is denoted as \TE{NodeConnectivityConstraint}, which ensures that each movable node is connected to exactly two other nodes with lower indices.} As a result, the movable node and the two other ones will form a triangle and the position of the movable node can then be determined via the law-of-cosine. We introduce auxiliary variables $C_{ji}^1$ to indicate whether $\NN_j$ is the first node to which $\NN_i$ is connected. $C_{ji}^2$ indicates whether $\NN_j$ is the second node to which $\NN_i$ is connected. In addition, we introduce two verbose variables $C_{0i}^{1,2}=1$ to indicate that $\NN_i$ is connected to nothing, which is the case when $\NN_i$ is fixed or unused. The resulting constraint set is:
\begin{equation}
\begin{aligned}
\label{eq:connectivity}
&C_{ji}^1,C_{ji}^2\in\{0,1\} \quad\forall 1\leq j<i\leq K   \\
&C_{ji}=C_{ji}^1+C_{ji}^2\in[0,1]    \quad
C_{ji}^1\leq U_j\land C_{ji}^2\leq U_j   \\
&\sum_{j=1}^{i-1} C_{ji}=2-2F_i \quad\forall 2\leq i\leq K \\
&C_{0i}^d\in\{0,1\} \quad \sum_{j=0}^{i-1} C_{ji}^d=1.% \quad \forall d=1,2,
\end{aligned}
\end{equation}
When $\NN_i$ is fixed in the above formulation, then $F_i=1$ in \prettyref{eq:connectivity} and all $C_{ji}$ are zero except for $C_{0i}^{1,2}=1$ due to the sum-to-one constraints. If $\NN_i$ is movable, then $F_i=0$ and $C_{ji}$ sums to two. As a result, there must be $j_1,j_2<i$ such that $C_{j_1i}^1=1$ and $C_{j_2i}^2=1$. Note that $j_1$ and $j_2$ must be different because otherwise the constraint that $C_{ji}\in[0,1]$ will be violated. In addition, since the first node $\NN_1$ is the motor node, it is excluded from these connectivity constraints. 

\refined{Our third set of constraints is denoted as \TE{NoWasteConstraint}, which ensures that the linkage structure contains no wasted parts.} In other words, each node must have some influence on the trajectory of the end-effector node and the first motor node must be connected to others. We model these constraints using the MICP formulation of network flows \citep{conforti2009network}. Specifically, each node $\NN_i$ will generate a flux that equals to $U_i$, and we assume that there is a flow edge defined between each pair of nodes with capacity $Q_{ji}$. We require inward-outward flux balance for each node except the end-effector node:
\begin{equation}
\begin{aligned}
\label{eq:balance}
&Q_{ji}\geq0 \quad\forall 1\leq j<i\leq K    \\
&Q_{ji}\leq C_{ji}K \\
&U_i+\sum_{j=1}^{i-1} Q_{ji}=\sum_{k=i+1}^{K} Q_{ik} 
\quad\forall 1\leq i\leq K-1,
\end{aligned}
\end{equation}
\refined{where the lefthand (resp. righthand) side of the last equation equals to the inward (resp. outward) flux. The inward flux is a sum of the newly generated flux $U_i$ and the flux passed on from nodes with lower-indices. Note that the flux-balance condition is imposed on every node except for the end-effector node, which means that only the end-effector node can deplete fluxes. As a result, every node must be connected to the end-effector node in order for its newly generated flux $U_i$ to be depleted. We illustrate one solution of $U_i,Q_{ji}$ in \prettyreft{fig:linkage} (a).} Here we adopt the big-M method \citep{bertsimas1997introduction} in the second constraint to ensure that only edges between connected nodes can have a capacity up to $K$. Big-M is a well-known method in mixed-integer modeling for choosing one element from a discrete set, or for choosing one case from several possible cases. 

\refined{Finally, using a similar idea, we also formulate a constraint that restricts a movable node to be connected to at least one other movable node (otherwise the movable node never moves), which is denoted as \TE{MovableNodeConstraint}.} We assume that each node $\NN_i$ generates a reversed outward flux that equals to $1-F_i$ and there is a flow edge defined between each pair of nodes with capacity $R_{ji}$. We require an inward-outward flux balance for each node except for the motor node:
\begin{equation}
\begin{aligned}
\label{eq:balance2}
&R_{ji}\geq0 \quad\forall 1\leq j<i\leq K   \\
&R_{ji}\leq C_{ji}K\quad R_{ji}\leq (1-F_j)K    \\
&\sum_{j=1}^{i-1} R_{ji}=1-F_i+\sum_{k=i+1}^{K} R_{ik} \quad\forall 2\leq i\leq K.
\end{aligned}
\end{equation}
These four constraints ensure that the planar linkage structure is symbolically correct, independent of the concrete geometric shape. In summary, our topological constraints involve:
\begin{equation*}
\begin{aligned}
&\text{Topological Constraints}\triangleq\\
&\begin{cases}
\prettyreft{eq:state}: \text{NodeUsageConstraint}\\
\prettyreft{eq:connectivity}: \text{NodeConnectivityConstraint}\\
\prettyreft{eq:balance}: \text{NoWasteConstraint}\\
\prettyreft{eq:balance2}: \text{MovableNodeConstraint}\\
\end{cases}.
\end{aligned}
\end{equation*}
A planar linkage satisfies \prettyref{ass:1} if topological constraints are satisfied. All the equations in this section are mixed-integer linear constraints that can be satisfied using an off-the-shelf MICP solver such as \cite{gurobi} as long as a solution exists.

\subsection{Reducing Binary Variables}
MIP solvers build a search tree by branching on binary variables whose continuous relaxation is not exact, so the size of the search tree and the performance of MIP solvers is closely related to the number of binary variables. Altogether, our topological constraints use $4K+K(K+1)$ binary variables, where the quadratic term comes from $C_{ji}^1,C_{ji}^2$. We can further reduce the number of variables to $\BO(K\lceil\log K\rceil)$ by adopting the idea of a special ordered set of type 1 ($\SOSA$) \citep{vielma2011modeling}. Intuitively, $\SOSA$ is a constraint that only one out of a set of $K$ variables can take a non-zero value. Their main idea is to order these variables from $1$ to $K$ and choose a number within this range. To this end, a binary variable is introduced to indicate whether each binary bit is $1$.  We can apply this idea to \prettyref{eq:connectivity} by observing that the two constraints $C_{ji}^d\in\{0,1\}$ and $\sum_{j=0}^{i-1}C_{ji}^d=1$ is equivalent to:
\begin{align*}
&\{C_{ji}^d|j=0,\cdots,i-1\}\in\SOSA\\
&\sum_{j=0}^{i-1}C_{ji}^d=1.
\end{align*}
The $\SOSA$ constraints can be converted to conventional linear constraints via \prettyref{alg:SOSA}, where we introduce at most $\lceil\log K\rceil$ auxiliary binary variables denoted as: $\mathbb{I}_1,\cdots,\mathbb{I}_{\lceil\log K\rceil}$.
\begin{algorithm}[ht]
\caption{Constraint $\{C_{ji}^d|j=0,\cdots,i-1\}\in\SOSA$}
\label{alg:SOSA}
\begin{algorithmic}[1]
\Require{$\mathbb{I}_1,\cdots,\mathbb{I}_{\lceil\log i\rceil}\in\{0,1\}$}
\For{$j=0,\cdots,i-1$}
\For{bit=$1,\cdots,\lceil\log i\rceil$}
\If{$j\land2^{\text{bit}-1}=0$}\Comment{Bitwise and}
\State Add constraint $C_{ji}^d\leq\mathbb{I}_\text{bit}$ 
\Else
\State Add constraint $C_{ji}^d\leq1-\mathbb{I}_\text{bit}$ 
\EndIf
\EndFor
\EndFor
\end{algorithmic}
\end{algorithm}

\subsection{Geometric Constraints}
\refined{We introduce three sets of geometric constraints to ensure that our linkage structure can be fabricated and generate unique end-effector trajectories. Our first constraint set is denoted as \TE{RealizabilityConstraint}, which ensures that the node positions at all time instances $t$ can be realized by the same set of links, so that the linkage structure can be fabricated.} For a pair of nodes $\NN_i$ and $\NN_j$ ($j<i$), there might be a link connecting them as indicated by the variable $C_{ji}^d$. \prettyref{eq:connectivity} dictates that there is a link between the two nodes if and only if $\sum_{d=1}^2C_{ji}^d=1$. Further, we ensure that the node never moves if $F_i=1$. Put together, our realizability constraint takes the following form:
\begin{align*}
&\|\NN_i(t_1)-\NN_j(t_1)\|^2=\|\NN_i(t_2)-\NN_j(t_2)\|^2\\
&\forall 0\leq t_1<t_2\leq2\pi\land \sum_{d=1}^2C_{ji}^d=1\\
&\|\NN_i(t_1)-\NN_i(t_2)\|=0
\quad\forall 0\leq t_1<t_2\leq2\pi\land F_i=1.
\end{align*}
We need to absorb the decision variable $C_{ji}^d$ into the constraint to be handled by the optimizer, to which end the big-M method can be used to derive the following equivalent constraints:
\begin{align}
\label{eq:length}
&\|\NN_i(t)-\NN_j(t)-\E{d}_{di}(t)\|\leq(1-C_{ji}^d)2\sqrt{2}B
\quad\forall 2\leq i\leq K\nonumber\\
&\|\E{d}_{di}(t_1)\|^2=\|\E{d}_{di}(t_2)\|^2
\quad\forall 0\leq t_1<t_2\leq2\pi  \\
&\|\NN_i(t_1)-\NN_i(t_2)\|\leq(1-F_i)2\sqrt{2}B
\quad\forall 0\leq t_1<t_2\leq2\pi.\nonumber
\end{align}
The big-M method requires an upper bound on the length difference between the two nodes at different time instances. If we assume the planar linkage is bounded inside a box with side length $B$, then a conservative upper bound is twice the diagonal length $2\sqrt{2}B$. Here we introduce the slack variables $\E{d}_{di}(t)$ that are constrained to be equal to $\NN_i(t)-\NN_j(t)$ when $C_{ji}^d=1$, in which case the equal-length constraint is specified for $\E{d}_{di}(t)$ instead. Otherwise, when $\NN_i$ is fixed, $\E{d}_{di}(t)$ can take any value and the equal-length constraint can be trivially satisfied.

\refined{Our second constraint set ensures that the orientations of nodes are unambiguous during the entire limit cycle and a linkage structure generates a unique end-effector trajectory. This is denoted as \TE{AreaConstraint}.} According to \prettyref{eq:forward}, there are two possible positions for a node $\NN_i$ corresponding to the sign of off-diagonal terms in the rotation matrix $\E{R}$. To ensure that the sign of the two off-diagonal terms never changes, it is enough to bound their values away from zero. Equivalently, we can bound the cosine value away from $1$ by a small margin denoted as $\epsilon$, i.e., $\cos\leq1-\epsilon$.  After some rearrangement, we derive the following equivalence:
\begin{align*}
&|\|\NN_j(t)-\NN_k(t)\|-l_{ji}|^2\leq l_{ki}^2-2\epsilon\|\NN_j(t)-\NN_k(t)\|l_{ji}.
\end{align*}
The above constraint is not a quadratic form, but an equivalent form exists by observing that $\cos=1$ if and only if the area of the triangle formed by $\NN_i,\NN_j,\NN_k$ has zero area. Therefore, we can bound the signed area away from zero as follows:
\begin{equation}
\begin{aligned}
\label{eq:area}
\left<\left[\E{d}_{1i}(t)\right]^\perp,\E{d}_{2i}(t)\right>
\geq\epsilon,
\end{aligned}
\end{equation}
where the $\perp$ superscript denotes the vector rotated by $90$ degrees clockwise. Note that the sign of the area does not matter because they correspond to two binary variable assignments $C_{ji}^1=C_{ki}^2=1$ and $C_{ki}^1=C_{ji}^2=1$, and our optimizer is free to choose one of the two cases. Additionally, the area constraint on the triangle between $\NN_i,\NN_j,\NN_k$ should only be activated when one of the two binary variable assignments happens, but we do not need to consider this issue here as it has been done in the first constraint of \prettyref{eq:length}.

\refined{We introduce the last set of constraints, denoted as \TE{MotorConstraint}, to specify the rotational motion of the first node $\NN_1$.} The motion specified by \prettyref{eq:motor} must be replaced with a quadratic form to be consumed by MIQCQP. Note that the motor can rotate clockwise or counter-clockwise. Therefore, we introduce a binary variable $D$ to distinguish between these two cases. Our constraint set is then formulated as:
\begin{equation}
\begin{aligned}
\label{eq:order}
&\|\E{R}(t)\E{d}_{d1}(t)-\E{d}_{d1}(0)\| \leq D2\sqrt{2}B   \\
&\|\E{R}(-t)\E{d}_{d1}(t)-\E{d}_{d1}(0)\| \leq (1-D)2\sqrt{2}B  \\
&\E{d}_{d1}(t)=\NN_1(t)-\TWO{X_C}{Y_C},\\
\end{aligned}
\end{equation}
where we have used the same upper bound for the big-M method. In summary, our geometric constraints involve:
\begin{equation*}
\begin{aligned}
&\text{Geometric Constraints}\triangleq\\
&\begin{cases}
\prettyreft{eq:length}: \text{RealizabilityConstraint}\\
\prettyreft{eq:area}: \text{AreaConstraint}\\
\prettyreft{eq:order}: \text{MotorConstraint}\\
\end{cases}.
\end{aligned}
\end{equation*}
\refined{We formalized the following result which shows that our linkage structure is bounded away from singularities under geometric constraints:}
\begin{lemma}
\label{lem:nonsingular}
\refined{If \prettyreft{eq:length}, \prettyreft{eq:area}, and \prettyreft{eq:order} are satisfied, then for any $t$, the linkage structure has no singular configurations of any types \citep{56660}, and the end-effector trajectory is unique.}
\end{lemma}
\begin{proof}
\refined{The forward kinematics of our linkage can be computed by solving the implicit equations:
\begin{align*}
\E{R}(\pm t)\E{d}_{d1}(t)-\E{d}_{d1}(0)=0&\\
\begin{rcases*}
\|\NN_i-\NN_j\|^2-l_{ji}^2=0\\
\|\NN_i-\NN_k\|^2-l_{ki}^2=0
\end{rcases*}&\forall\CNN{i}{j}{k},
\end{align*}
all the equations of which hold by \prettyreft{eq:length} and \prettyreft{eq:order}. We can summarize the above equations as a vector-valued implicit function $\E{F}(t,\NN)=0$, where $\NN$ is a concatentation of movable node positions. Only the first two rows of $\E{F}$ correspond to the motor node $\NN_1$ and each non-motor, movable node $\CNN{i}{j}{k}$ occupies two additional rows. The sensitivity analysis leads to:
\begin{align*}
\FPP{\E{F}}{t}\dot{t}+\FPP{\E{F}}{\NN}\dot{\NN}=0.
\end{align*}
Since only the first two rows are functions of $t$, we immediately have $\FPPR{\E{F}}{t}$ is a single non-zero column, which implies that the linkage does not have singular configuration of type (i) or (iii) \citep{56660}. Next, we focus on $\FPPR{\E{F}}{\NN}$, which is a square matrix. This is because each (motor or non-motor), movable node will introduce two rows and two columns. We first order the rows and columns of $\E{F}$ in the topological ordering as implied by \prettyref{ass:1}, then $\FPPR{\E{F}}{\NN}$ becomes a $2\times2$ block upper-triangular matrix of the following type:
\begin{align*}
\FPP{\E{F}}{\NN}=\MTT
{\ddots}{\E{0}}{\ldots}
{\MTT
{2(X_i-X_j)}{2(Y_i-Y_j)}
{2(X_i-X_k)}{2(Y_i-Y_k)}},
\end{align*}
from which we immediately have $\det(\FPPR{\E{F}}{\NN})\geq (4\epsilon)^{\bar{N}-1}>0$ by \prettyreft{eq:area} and the linkage does not have singular configuration of type (ii) \citep{56660}, where $\bar{N}$ denotes the number of movable nodes. Since $\det(\bullet)$ is a continuous function, the linkage structure is bounded away from singularity and the end-effector trajectory is unique.\hfill\qedsymbol}
\end{proof}
The geometric constraint set consists of an infinite number of quadratic constraints. Unfortunately, the two crucial constraints (i.e.: the equal-length constraints in \prettyref{eq:length} and \prettyref{eq:area}) are non-convex, so finding the feasible solution set for them is not tractable. In the next section, we propose methods to discretize and then approximate its optimal solution set.

%% file: solver.tex
\section{Approximate MIQCQP Solver}
The above-mentioned MIQCQP involves infinitely many variables, intractable integral in the objective function, and non-convex constraints. To derive a finite-dimensional problem, we discretize our trajectory, as well as the user-specified target trajectory, by sampling $T$ nodes evenly at $t^q=2\pi q/T$ with $q=1,\cdots,T$. Under such discretization, our objective function can be approximated as:
\begin{align}
\label{eq:objectiveDiscrete}
&\int_0^{2\pi}\|\NN_N(t)-\NN_N^*(t)\|^2dt+\text{reg.}\nonumber\\
\approx&\frac{2\pi}{T}\sum_{q=1}^T\|\NN_N(t^q)-\NN_N^*(t^q)\|^2+\lambda\sum_{i=1}^NU_i,
\end{align}
where we have utilized the variable $U_i$ as our regularization term, which encourages the optimizer to use as few links as possible, and $\lambda$ is the weight coefficient that balances the exactness of the target curve matching and the simplicity of the linkage structure. Similarly, we can discretize the geometric constraint \prettyref{eq:length} as:
\begin{align}
\label{eq:lengthDiscrete1}
&\|\NN_i(t^q)-\NN_j(t^q)-\E{d}_{di}(t^q)\|
\leq(1-C_{ji}^d)2\sqrt{2}\\
\label{eq:lengthDiscrete2}
&\|\E{d}_{di}(t^q)\|^2=\|\E{d}_{di}(t^{q+1})\|^2\\
\label{eq:lengthDiscrete3}
&\|\NN_i(t^q)-\NN_i(t^{q+1})\|^2\leq(1-F_i)2\sqrt{2}B,
\end{align}
\prettyref{eq:area} as:
\begin{equation}
\begin{aligned}
\label{eq:areaDiscrete}
\left<\left[\E{d}_{1i}(t^q)\right]^\perp,\E{d}_{2i}(t^q)\right>
\geq\epsilon,
\end{aligned}
\end{equation}
and finally \prettyref{eq:order} as:
\begin{equation}
\begin{aligned}
\label{eq:orderDiscrete}
&\|\E{R}( t^q)\E{d}_{d1}(t^q)-\E{d}_{d1}(0)\| \leq D2\sqrt{2}B   \\
&\|\E{R}(-t^q)\E{d}_{d1}(t^q)-\E{d}_{d1}(0)\| \leq (1-D)2\sqrt{2}B    \\
&\E{d}_{d1}(t^q)=\NN_1^d(t^q)-\TWO{X_C}{Y_C}.
\end{aligned}
\end{equation}
In summary, our discrete MIQCQP takes the following form:
\begin{equation}
\begin{aligned}
\label{eq:MIQCQPDiscrete}
\argmin{}&\prettyreft{eq:objectiveDiscrete}\\
\argcon{}&\prettyreft{eq:state},\ref{eq:connectivity},\ref{eq:balance},\ref{eq:balance2},
\ref{eq:lengthDiscrete1},\ref{eq:lengthDiscrete2},\ref{eq:lengthDiscrete3},
\ref{eq:areaDiscrete},\ref{eq:orderDiscrete}.
\end{aligned}
\end{equation}
The constraint set of \prettyref{eq:MIQCQPDiscrete} is a subset of \prettyref{eq:MIQCQP}, which defines an outer approximation for the feasible set of \prettyref{eq:MIQCQP}. Specifically, a feasible solution for \prettyref{eq:MIQCQPDiscrete} might not be realizable or satisfy the area constraints in between two time samples, but such discretization error can be made arbitrarily close to zero as $T\to\infty$. Although \prettyref{eq:MIQCQPDiscrete} is a finite-dimensional problem, finding its global optima is still intractable due to the non-convex constraints. In this section, we propose three algorithms to approximate its solution. 

\input{MICP.tex} %reviewed
\input{MINLP.tex} %reviewed
\input{SA.tex} %reviewed

%% file: MICP.tex
\subsection{Mixed-Integer Conic Programming}
Our first algorithm uses a convex outer approximation for each non-convex quadratic constraint and then uses off-the-shelf MICP solver such as \cite{gurobi} to find approximate solutions. To relax the non-convex constraint \prettyref{eq:lengthDiscrete2}, we borrow techniques from \cite{liberti2004reformulation}. We notice that the non-convex constraint can be written as a linear constraint as the square of the node coordinates $\left[X_{di}\right]^2$, $\left[Y_{di}\right]^2$:
\begin{align*}
&\|\E{d}_{di}(t^q)\|^2=\|\E{d}_{di}(t^{q+1})\|^2\Leftrightarrow\\
&\left[X_{di}\right]^2(t^q)    +\left[Y_{di}\right]^2(t^q)=
\left[X_{di}\right]^2(t^{q+1})+\left[X_{di}\right]^2(t^{q+1}),
\end{align*}
where we assume $\E{d}_{di}=\TWO{X_{di}}{Y_{di}}$. For any decision variable $\alpha$, \cite{liberti2004reformulation} proposed a technique to derive a piecewise linear upper bound of $\alpha^2$, denoted as $\alpha^2\leq\tilde\alpha$, and the approximation error ($\tilde\alpha-\alpha^2$) can be made arbitrarily small by using more pieces. As illustrated in \prettyref{fig:upperBound}, the upper bound is formed by evenly sampling $S+1$ points on the curve $\alpha^2$ and then connect the samples using straight line segments. If we know that $\alpha\in[-B/2,B/2]$, then the sample points are $\alpha_s=sB/S-B/2$, where $s=0,\cdots,S$. To constraint $\tilde\alpha$ to lie on the set of line segments, we use the following set of constraints:
\begin{equation}
\begin{aligned}
\label{eq:upper}
&\TWOC{\alpha}{\tilde{\alpha}}=\sum_{s=0}^S\lambda_s\TWOC{\alpha_s}{\alpha_s^2} \\
&\{\lambda_{0,\cdots,S}\}\in\SOSB
\quad\sum_{s=0}^S\lambda_s=1,
\end{aligned}
\end{equation}
where we have used auxiliary variables $\lambda_s$ that belong to the special ordered set of type 2 ($\SOSB$) \citep{vielma2011modeling}. $\SOSB$ requires that at most two of the variables in an ordered set with consecutive indices can take non-zero values. The $\SOSB$ constraint can be converted to a set of linear constraints and a $\SOSA$ constraint using \prettyref{alg:SOSB}. With the upper bound, we can approximate the non-convex constraint with two linear constraints:
\begin{equation}
\begin{aligned}
\label{eq:lengthRelaxed}
&\|\E{d}_{di}(t^q)\|^2\leq\tilde{X}_{di}(t^{q+1})+\tilde{Y}_{di}(t^{q+1})\\
&\|\E{d}_{di}(t^{q+1})\|^2\leq\tilde{X}_{di}(t^q)+\tilde{Y}_{di}(t^q).
\end{aligned}
\end{equation}
It can be shown that \prettyref{eq:lengthRelaxed} forms an outer approximation for the feasible region of the non-convex constraint and the approximation error diminishes as $S\to\infty$. A similar formulation has been used in \cite{dai2017global} to discretize the space of unit vectors. To formulate \prettyref{eq:lengthRelaxed}, we need $4TK$ upper bounds, each one introducing $\lceil\log S\rceil$ binary decision variables, so we need $4TK\lceil\log S\rceil$ binary variables altogether.
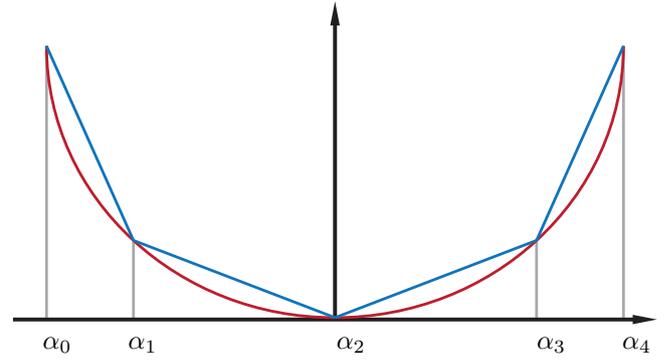
\begin{figure}[t]
\includegraphics[width=0.49\textwidth]{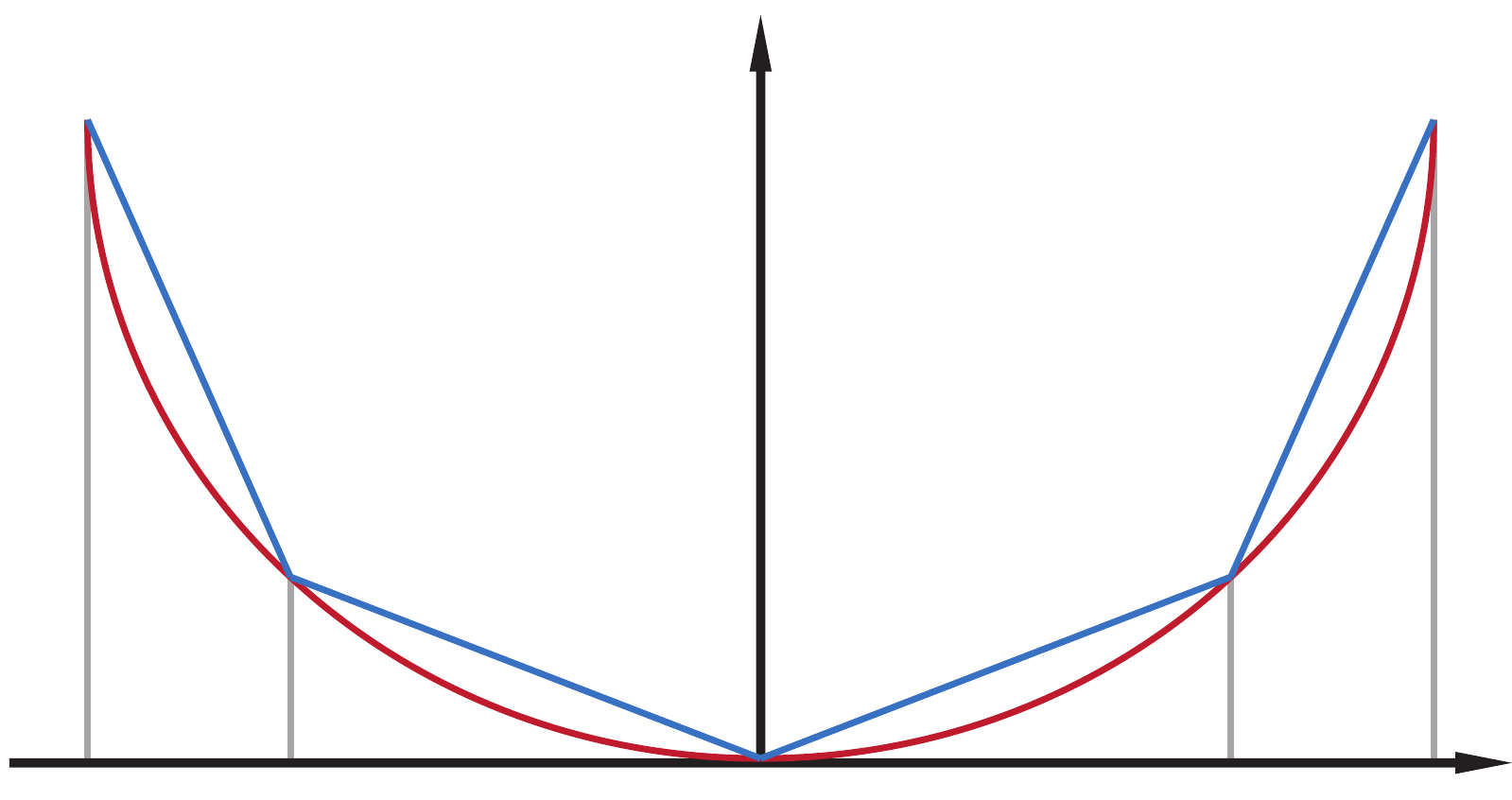}
\put(-230,-5){$\alpha_0$}
\put(-198,-5){$\alpha_1$}
\put(-120,-5){$\alpha_2$}
\put(-45,-5){$\alpha_3$}
\put(-13,-5){$\alpha_4$}
\caption{\label{fig:upperBound} An illustration of the piecewise linear upper bound (blue) of the quadratic curve $\alpha^2$ (red) with $S=4$.}
\end{figure}
\begin{algorithm}[ht]
\caption{Constraint $\{\lambda_0,\cdots,\lambda_S\}\in\SOSB$}
\label{alg:SOSB}
\begin{algorithmic}[1]
\Require{Auxiliary variables $\bar\lambda_0,\cdots,\bar\lambda_{S-1}\in[0,1]$}
\Require{$\bar\lambda_{-1}=\bar\lambda_{S}=0$}
\State Add constraint $\{\bar\lambda_0,\cdots,\bar\lambda_{S-1}\}\in\SOSA$\Comment{\prettyref{alg:SOSA}}
\For{$i=0,\cdots,S$}
\State Add constraint $\lambda_i\leq\bar\lambda_{i-1}+\bar\lambda_i$
\EndFor
\end{algorithmic}
\end{algorithm}

\begin{figure*}[ht]
%\vspace{-25px}
\centering
\scalebox{0.9}{
\includegraphics[width=0.99\textwidth]{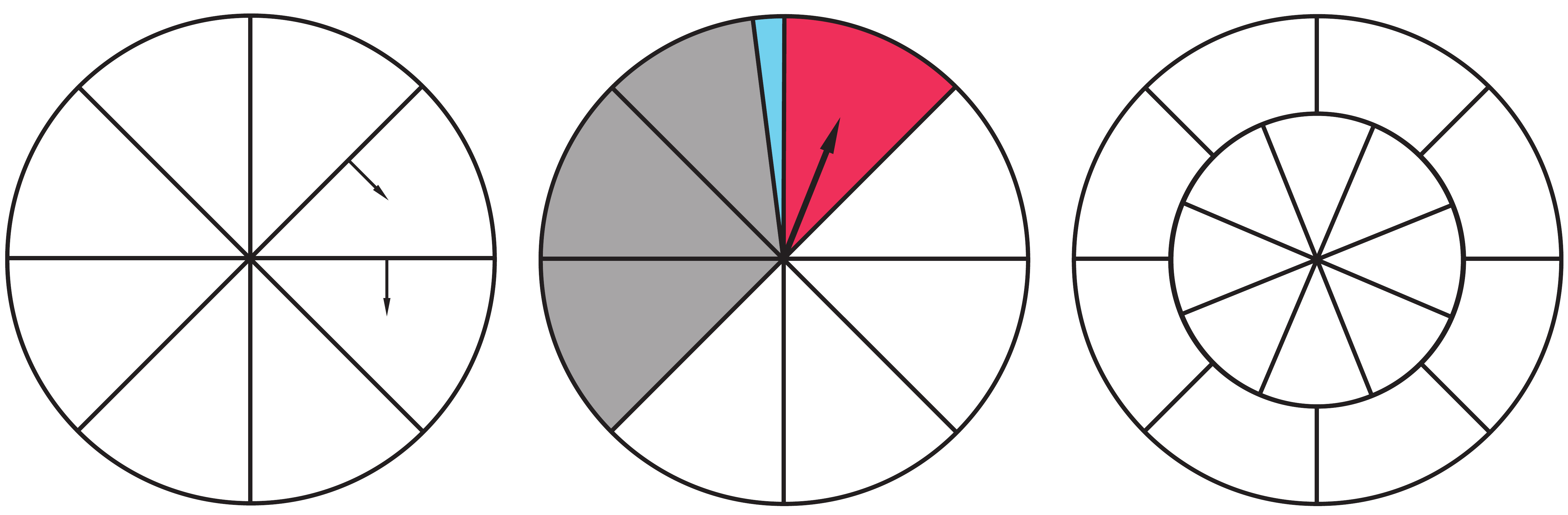}
%a
\put(-385,90){$\E{v}_1^L$}
\put(-386,63){$\E{v}_1^R$}
\GAMMA{0}{417}{78}{60}{$\gamma_1$}
\GAMMA{1}{417}{78}{60}{$\gamma_2$}
\GAMMA{2}{417}{78}{60}{$\gamma_3$}
\GAMMA{3}{417}{78}{60}{$\gamma_4$}
\GAMMA{4}{417}{78}{60}{$\gamma_5$}
\GAMMA{5}{417}{78}{60}{$\gamma_6$}
\GAMMA{6}{417}{78}{60}{$\gamma_7$}
\GAMMA{7}{417}{78}{60}{$\gamma_8$}
\put(-360,10){(a)}
%b
\put(-252,160){$\epsilon$}
\put(-225,125){$\E{d}_{1i}$}
\put(-300,95){$\E{d}_{2i}$}
\put(-190,10){(b)}
%c
\GAMMA{0}{85}{78}{60}{$\gamma_1$}
\GAMMA{1}{85}{78}{60}{$\gamma_2$}
\GAMMA{2}{85}{78}{60}{$\gamma_3$}
\GAMMA{3}{85}{78}{60}{$\gamma_4$}
\GAMMA{4}{85}{78}{60}{$\gamma_5$}
\GAMMA{5}{85}{78}{60}{$\gamma_6$}
\GAMMA{6}{85}{78}{60}{$\gamma_7$}
\GAMMA{7}{85}{78}{60}{$\gamma_8$}
\GAMMAI{0}{85}{78}{60}{$\gamma_9$}
\GAMMAI{1}{85}{78}{60}{$\gamma_{10}$}
\GAMMAI{2}{85}{78}{60}{$\gamma_{11}$}
\GAMMAI{3}{85}{78}{60}{$\gamma_{12}$}
\GAMMAI{4}{85}{78}{60}{$\gamma_{13}$}
\GAMMAI{5}{85}{78}{60}{$\gamma_{14}$}
\GAMMAI{6}{85}{78}{60}{$\gamma_{15}$}
\GAMMAI{7}{85}{78}{60}{$\gamma_{16}$}
\put(-20 ,10){(c)}}
%\vspace{-0px}
\caption{\label{fig:angleConstraint} Linear relaxation of angle constraints. (a): We cut $\mathcal{SO}(2)$ into 8 sectors, each of which is selected by a $\gamma$-flag. The sector selected by $\gamma_1$ is bounded by its left/right normal vectors $\E{v}_1^L$/$\E{v}_1^R$. (b): If $\E{d}_{1i}$ falls in the red area, then we restrict $\E{d}_{2i}$ to the gray area, which is at least $\epsilon$-apart (blue). However, when $\E{d}_{1i}$ moves across sector boundaries, the gray area jumps discontinuously. (c): To avoid discontinuous changes for $\E{d}_{2i}$ when $\E{d}_{1i}$ undergoes continuous changes, we propose to double cover $\mathcal{SO}(2)$.}
%\vspace{-15px}
\end{figure*}
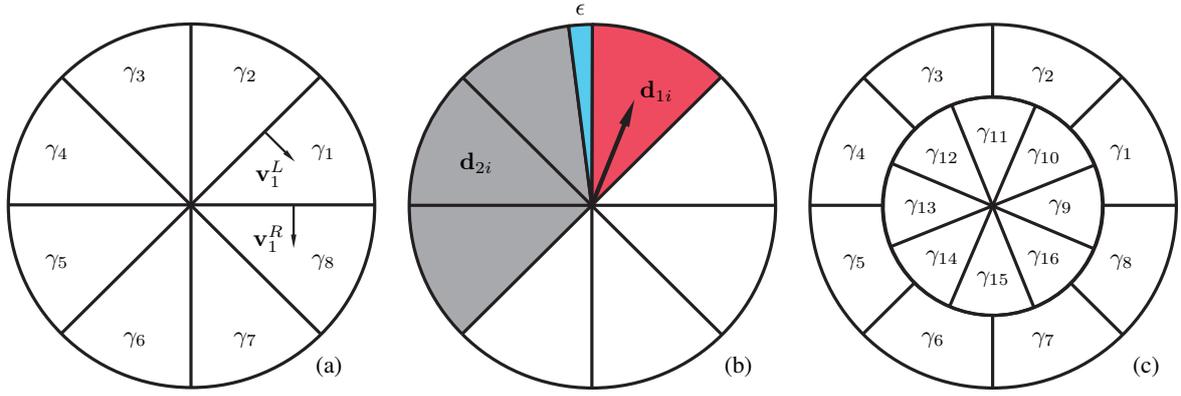

We adopt a similar technique to relax the area constraints \prettyref{eq:areaDiscrete}. These constraints involve two bilinear terms:
\begin{align*}
&\left<\left[\E{d}_{1i}(t^q)\right]^\perp,\E{d}_{2i}(t^q)\right>\geq\epsilon\Leftrightarrow\\
&X_{1i}(t^q)Y_{2i}(t^q)-Y_{1i}(t^q)X_{2i}(t^q)\geq\epsilon,
\end{align*}
which contributes to the non-convexity. A standardized technique to relax bilinear constraints is the McCormick envelop (see \cite{liberti2004reformulation} for more details), where the feasible domain is outer-approximated as a union of convex hulls. However, unlike for length constraints, we argue that inner approximations should be used for area constraints. This is because the parameter $\epsilon$ is a small constant and the McCormick envelope would also introduce relaxation errors with a larger magnitude than $\epsilon$, so even if the area constraints after the McCormick relaxation are satisfied, the exact constraints are still violated, rendering the relaxation useless. Instead, we propose an inner-approximation scheme. We cut the 2D rotation group $\mathcal{SO}(2)$ into $S$ sectors, as illustrated in \prettyref{fig:angleConstraint} (a), so that $\E{d}_{1i}$ will only fall into one of the $S$ sectors. If $\E{d}_{1i}$ falls in a particular sector, then we restrict $\E{d}_{2i}$ to its left half-space that is at least $\epsilon$ degrees apart, as shown in \prettyref{fig:angleConstraint} (b). If we use an $\SOSA$ constraint to select the sector in which $\E{d}_{1i}$ falls, then only $\BO(TK\lceil\log S\rceil)$ binary decision variables are needed. A minor issue with this formulation is that the allowed region of $\E{d}_{2i}$ jumps discontinuously as $\E{d}_{1i}$ changes continuously. We can fix this problem by double-covering the region of $\mathcal{SO}(2)$ by using $2S$ sectors, as shown in \prettyref{fig:angleConstraint} (c). To formulate these constraints, we assume that each sector of $\mathcal{SO}(2)$ is flagged by a selector variable $\gamma_l$, which is bounded by its left/right unit-length plane-normal vectors $\E{v}_l^L$/$\E{v}_l^R$. Combined with the fact that constraints should only be satisfied for one particular sector and for only movable nodes, we have the following formulation:
\begin{equation}
\begin{aligned}
\label{eq:sectorFinal}
&<\E{v}_l^L,\E{d}_{1i}(t^q)>\geq \sqrt{2}B(\gamma_{l,i}(t^q)-1) \\
&<\E{v}_l^R,\E{d}_{1i}(t^q)>\leq \sqrt{2}B(1-\gamma_{l,i}(t^q)) \\
&<\E{R}(\epsilon)\E{v}_l^L,\E{d}_{2i}(t^q)>\leq \sqrt{2}B(1-\gamma_{l,i}(t^q))  \\
&<\E{R}(\pi)\E{v}_l^R,\E{d}_{2i}(t^q)>\geq\sqrt{2}B(\gamma_{l,i}(t^q)-1)  \\
&\{\gamma_{1,1i}(t^q),\cdots,\gamma_{2S,i}(t^q)\}\in\SOSA\quad
\sum_{l=1}^{2S}\gamma_{l,i}^d=1,
\end{aligned}
\end{equation}
where $TK\lceil\log(2S)\rceil$ binary variables are used. 

We conclude that MICP can solve the relaxed and discretized MIQCQP problem by using $\BO(K\lceil\log K\rceil+4TK\lceil\log S\rceil+TK\lceil\log(2S)\rceil)$ binary variables (by replacing \prettyref{eq:lengthDiscrete2} with \prettyref{eq:lengthRelaxed} and \prettyref{eq:areaDiscrete} with \prettyref{eq:sectorFinal}), where the first term is due to topological constraints, the second term due to relaxed length constraints, and the last term due to relaxed angle constraints. Furthermore, the error due to relaxation and discretization can be made arbitrarily small as $T,S\to\infty$. In practice, the solution of MICP does not satisfy the exact, non-convex constraints and we remedy the error by locally solving \prettyref{eq:MIQCQPDiscrete} while fixing all the binary variables using an NLP solver.

%% file: MINLP.tex
\subsection{Mixed-Integer Nonlinear Programming}
Despite the theoretical advantage of MICP, its practical performance can be unacceptable due to an excessive number of binary variables. For example, if we only use a coarse discretization and relaxation with $K=7,T=10,S=8$, the number of binary variables is already $1064$, for which finding the exact global optima is impossible and we have to terminate the optimization early and return users the first few feasible solutions. To make things worse, the returned solution might be useless by not satisfying the exact non-convex constraints. 

We observe that it is both impossible and unnecessary to solve the relaxed problem to get the global optima, because even the global optima might not satisfy the exact constraints under a coarse relaxation. Instead, it is worthwhile to return a solution that is not globally optimal but satisfies the exact, non-convex constraints. Our second algorithm uses an off-the-shelf MINLP solver, such as \cite{byrd2006k}, to find a locally optimal solution. Similar to MICP, an MINLP solver is also based on the BB algorithm and constructs a search tree by branching on binary decision variables. For each node, however, we do not assume that the problem is convex and use SQP to find a locally optimal, feasible solution as described in \cite{exler2007trust}. The node will be pruned for further expansion if no feasible solution can be found or the objective function is larger than the incumbent. Since SQP does not guarantee to find the global optima or even a feasible solution when one exists, MINLP might prune a node that contains useful solutions. On the other hand, all the returned solutions are guaranteed to satisfy the exact non-convex constraints. 

\begin{figure*}[ht]
%\vspace{-25px}
\centering
\scalebox{0.9}{\includegraphics[width=0.99\textwidth]{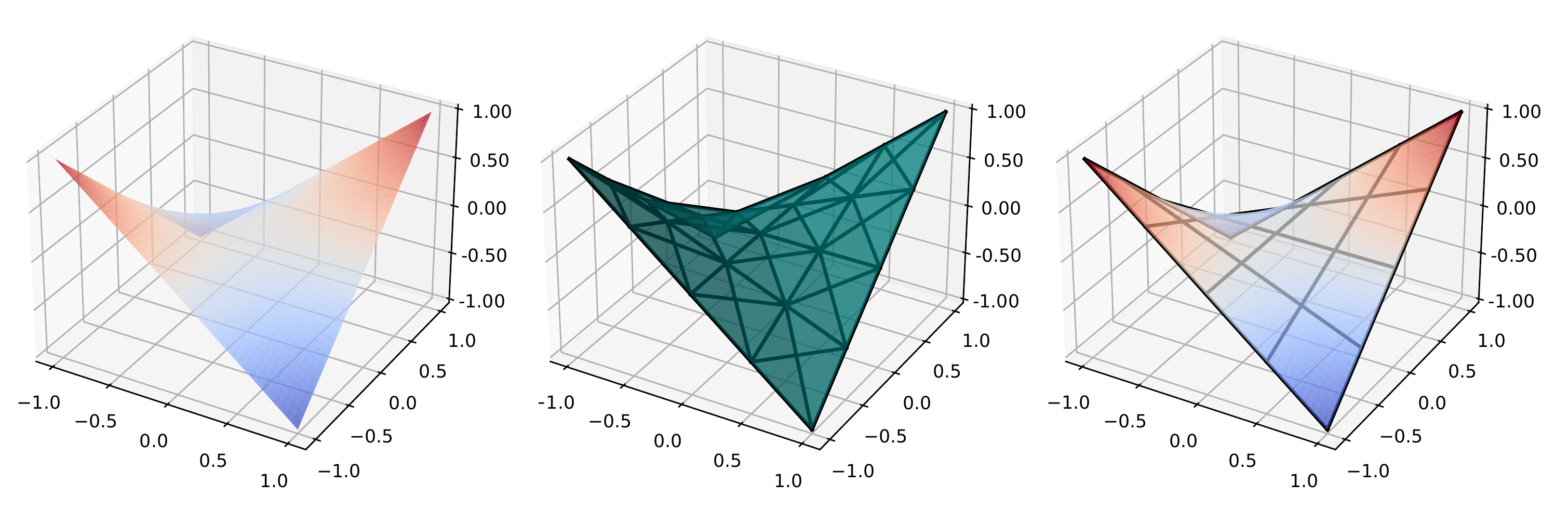}}
\put(-400,5){(a)}
\put(-255,5){(b)}
\put(-110,5){(c)}
%\vspace{-0px}
\caption{\label{fig:envelopeCompare} Taking the bilinear constraint $z=xy$ as an example, we compare the three different relaxation techniques. (a): The original constraint of MIQCQP is non-convex; (b): MICP approximates the feasible domain as the union of convex tetrahedra (teal); (c): MIQCQP keeps the non-convex constraint, but we introduce additional relaxations to force the solver to try multiple initial values that lie in each black block.}
%\vspace{-15px}
\end{figure*}
Although MINLP can directly handle \prettyref{eq:MIQCQPDiscrete} without any relaxation, we argue that a relaxation similar to MICP can also be used to force MINLP to try more initial points and increase the chance of finding better solutions. Take the equal length constraints for example, each $\E{d}_{di}\in[-B/2,B/2]^2$ by assumption, so we can evenly divide the domain into $S^2$ blocks. This can be done by introducing two sets of sample points $\alpha_s=sB/S-B/2$ and $\beta_s=sB/S-B/2$ and the following constraints:
\begin{equation}
\begin{aligned}
\label{eq:distanceRelaxation}
q=1: 
\begin{cases}
&\E{d}_{di}(t^q)=\TWO{\sum_{s=0}^S\lambda_s^X\alpha_s}{\sum_{s=0}^S\lambda_s^Y\beta_s}\\
&\{\lambda_{0,\cdots,S}^X\}\in\SOSB\quad\sum_{s=0}^S\lambda_s^X=1\\
&\{\lambda_{0,\cdots,S}^Y\}\in\SOSB\quad\sum_{s=0}^S\lambda_s^Y=1,
\end{cases}
\end{aligned}
\end{equation}
where $\lambda_s^{X,Y}$ are auxiliary variables. Adding these constraints will force the MINLP to insert nodes into the search tree corresponding to initializing $\E{d}_{di}$ in different blocks. In practice, we find that it is enough to relax only the equal length constraints and apply the relaxation only to the first timestep $q=1$. We conclude that MINLP can solve the discretized MIQCQP problem by using $\BO(K\lceil\log K\rceil+4K\lceil\log(S)\rceil)$ binary variables (by adding \prettyref{eq:distanceRelaxation} to \prettyref{eq:MIQCQPDiscrete}). This formulation uses much fewer binary variables (e.g., $98$ variables when $K=7,T=10,S=8$). Finally, we informally argue that if we apply relaxation to every $1\leq q\leq T$ and let $T,S\to\infty$, then MINLP will also find the global optima of MIQCQP because this is essentially asking MINLP to initialize from all possible solutions.

%% file: SA.tex
\subsection{Simulated Annealing}
We introduce our third algorithm as a baseline for comparison, which is based on the SA framework, similar to \cite{Zhu:2012:MMT:2366145.2366146,cabrera2002optimal} but adapted to our subclass of planar linkages satisfying \prettyref{ass:1}. We use the SA algorithm described in \cite{bertsimas1993simulated} to minimize our discrete objective function (\prettyref{eq:objectiveDiscrete}). Starting from a trivial initial guess $\mathcal{L}_0$, SA generates a Markov chain by mutating $\mathcal{L}_i$ to $\mathcal{L}_{i+1}$ at the $i$th iteration and accept $\mathcal{L}_{i+1}$ with a probability proportional to the decrease in the objective function. Our initial guess $\mathcal{L}_0$ involves only one motor node $\NN_1$ with $r=B/10,X_C=Y_C=0$. 

\begin{algorithm}[ht]
\caption{SA-Mutation}
\label{alg:SAMutation}
\begin{algorithmic}[1]
\Require $\mathcal{L}_i$
\State Succ$\gets$False
\While{Not Succ}
\State Choose move type, $\mathcal{L}_{i+1}\gets\mathcal{L}_i$
\If{type=TopologicalAdditionMove}
\State IsFixed$\sim\mathcal{U}(0,1)$
\If{IsFixed$>0.5$}
\State Choose $\NN_{N+1}\sim\mathcal{U}(-B/2,B/2)^2$
\State $N\gets N+1$, Succ$\gets N\leq K$
\Else{}
\State Choose $1\leq j<k\leq N$
\State Choose $l_{j(N+1)},l_{k(N+1)}\sim\mathcal{U}(0,B)$
\State Add node $\CNN{N+1}{j}{k},N\gets N+1$
\small
\State Succ$\gets N\leq K\land$
Satisfied(\ref{eq:state},\ref{eq:connectivity},\ref{eq:balance},\ref{eq:balance2})
$\land$Valid($\mathcal{L}_{i+1}$)
\normalsize
\EndIf
\ElsIf{type=TopologicalSubtractionMove}
\State Remove $\NN_N, N\gets N-1$
\State Succ$\gets N\geq1\land$
Satisfied(\ref{eq:state},\ref{eq:connectivity},\ref{eq:balance},\ref{eq:balance2})
\ElsIf{type=GeometricPerturbationMove}
\State Choose $1\leq i\leq N, t\sim \mathcal{U}(0,2\pi)$
\State Choose $\Delta X,\Delta Y\sim\mathcal{N}(B,B^2)$
\State $\NN_i(t)\gets\NN_i(t)+\TWO{\Delta X}{\Delta Y}$
\State Succ$\gets$Valid($\mathcal{L}_{i+1}$)
\ElsIf{type=LocalOptimizationMove}
\LineComment{The last term computed using \prettyref{alg:forward}}
\tiny
\State $d\gets\frac{2\pi}{T}\sum_{q=1}^T
\FPP{\|\NN_N(t^q)-\NN_N^*(t^q)\|^2}{\NN_N(t^q)}
\FPP{\NN_N(t^q)}{\FIVE{l_{ji}}{l_{ki}}{X_C}{Y_C}{r}}$
\normalsize
\State $\alpha\gets$Armijo-Line-Search($d$)
\tiny
\State $\FIVE{l_{ji}}{l_{ki}}{X_C}{Y_C}{r}\gets
\FIVE{l_{ji}}{l_{ki}}{X_C}{Y_C}{r}+\alpha d$
\normalsize
\EndIf
\EndWhile
\State Return $\mathcal{L}_{i+1}$
\end{algorithmic}
\end{algorithm}
We propose a mutation scheme outlined in \prettyref{alg:SAMutation} that chooses one of four moves with equal probability, where we use the function Valid($\bullet$) to check whether the linkage structure has valid kinematics (i.e., $\cos\in[0,1]$ for all $t\in[0,2\pi]$). \TE{TopologicalAdditionMove} adds a new node that can either be a fixed node or a movable node. A movable node $\CNN{N+1}{j}{k}$ is added by randomly selecting two existing nodes $\NN_j$, $\NN_k$ with $1\leq j<k\leq N$ and then selecting two length parameters $l_{j(N+1)},l_{k(N+1)}\sim\mathcal{U}(0,B)$ with uniform distribution. A fixed node is uniformly randomly selected within $[-B/2,B/2]^2$. We allow a TopologicalAdditionMove to happen if the total number of nodes is less than $K$, all the topological constraints are satisfied (if a movable node is added), and the planar linkage has valid kinematics. Our second \TE{TopologicalSubtractionMove} simply removes the last node $\NN_N$ from $\mathcal{L}_i$. We allow a TopologicalSubtractionMove to happen if there is at least one motor node remaining and all the topological constraints are satisfied. We check the topological constraints using our MICP solver after each topological move. Our third \TE{GeometricPerturbationMove} would first randomly select a node $\CNN{i}{j}{k}$ and then perturb its position at any time instance by adding a Gaussian noise $\mathcal{N}(B,B^2)$. We allow a GeometricPerturbationMove as long as the linkage structure has valid kinematics. Finally, we introduce a novel \TE{LocalOptimizationMove} that locally minimizes \prettyref{eq:objectiveDiscrete} with respect to all the geometric parameters (link lengths, $X_C,Y_C,r$) by using a gradient-based method. This is a computationally costly move that involves gradient evaluation using \prettyref{alg:forward} and then we choose to only perform a single steepest descend step with an Armijo backtracing line-search to ensure the decrease of objective function. Some of these moves might be unsuccessful, in which case we keep selecting a new type of move until one is successful.

%% file: result.tex
\section{Results}
We implement our three algorithms using Python, where we use the Supporting Hyperplane Optimization Toolkit (SHOT) \citep{kronqvist2016extended} to solve both MICP and MINLP problems, with IPOPT \citep{wchter:DSP:2009:2089} being the low-level NLP solver. \refined{Since the low-level NLP problems are non-convex, the solution of MINLP is sensitive to their initial guesses. Specifically, we use IPOPT in two phases, where phase-I ignores the objective function and only tries to satisfy the constraints and phase-II takes the objective functions into account. For each node of the BB search tree, we use the solution of phase-I of the parent node as an initialization.} Our SA-baseline is implemented based on the algorithm described in \cite{bertsimas1993simulated}. All the experiments are performed on a desktop computer with a 10-core Intel Xeon(R) W-2155 CPU. We use all the 10 cores to explore multiple nodes of the BB search tree in parallel. In this section, we discuss and compare our three algorithms in terms of computational cost, optimality, additional user constraints, and robot integration. 

\input{table.tex}
\subsection{Computational Cost and Optimality}
We designed a graphical user interface that allows a robot designer to sketch a curve. We then close the curve and evenly sample $T$ points on the curve such that we have equal arc lengths between two consecutive points. In \prettyref{table:results}, we illustrate $10$ testing target curves and the resulting planar linkages found by MICP, MINLP, and the SA-baseline. To get these results, we set $K=7$ for all three algorithms. We use $T=10, S=8$ for MICP, $T=20, S=8$ for MINLP. The SA-baseline requires a cooling function for the temperature $\mathcal{T}$, for which we use:
\begin{align*}
\mathcal{T}=\mathcal{T}_\text{max}\exp
(-\log(\frac{\mathcal{T}_\text{max}}{\mathcal{T}_\text{min}})\frac{i}{i_\text{max}}).
\end{align*}
We set $\mathcal{T}_\text{max}=2.5\times10^4$ and $\mathcal{T}_\text{min}=2.5$, $i_\text{max}=5\times10^4$ where $i$ is the iteration number. Due to the limited computational resources and time, we allocate a maximum $24$ hours running time for each test and return the best solution. We tune the SA-baseline such that its computational cost is comparable to MICP or MINLP, for which running $5\times10^4$ iterations takes approximately $1$ hour and running $5\times10^5$ iterations takes roughly $10$ hours (one iteration corresponds to one execution of an SA-Mutation \prettyref{alg:SAMutation}). We set up the parameters for our three algorithms, so that their computational times are roughly comparable. According to \prettyref{table:results}, the SA-baseline cannot generate satisfactory results in terms of matching the target curve. We notice that this is not because of the parameter settings of the SA-baseline. Indeed, the results are still unsatisfactory if we increase the iterations number $10$ times, as shown in the rightmost column of \prettyref{table:results}. \refined{We further highlight that a larger $S$ could lead to better solutions. As illustrated in \prettyref{fig:LargeS}, we fix the parameters $K=5, T=20$, while comparing the objective function values under three choices: $S=1,4,8$. Using a larger $S$ reduce the objective function value for $10$ out of $11$ examples. In one of the example, the improvement can be as high as $78\%$.}

\begin{figure}[ht]
\centering
\includegraphics[height=0.3\textwidth]{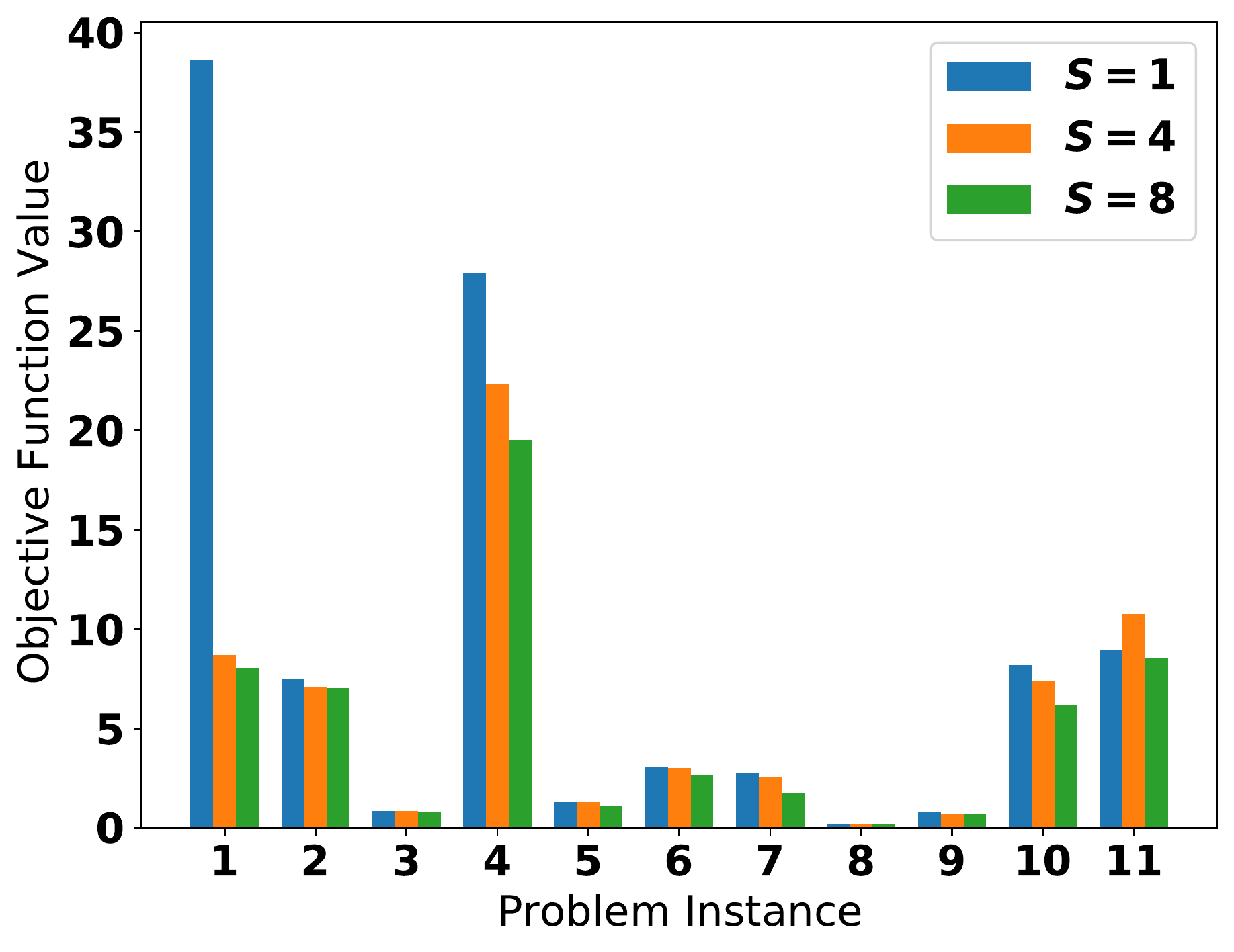}
\caption{\label{fig:LargeS} We compared the objective function values with three difference choices of $S=1,4,8$, while using $K=5$, $T=20$. Using a larger $S$ improves the solution for $10$ out of $11$ examples. In most examples, the improvement is minor, but the largest improvement can be as high as $78\%$.}
\vspace{-5px}
\end{figure}
We further analyze the computational cost and the convergence history of the MICP and MINLP algorithms as a function of $K$ and $S$, as shown in \prettyref{fig:costFunction} and \prettyref{fig:convergence}, respectively. We use both algorithms to solve ten benchmark problems under different parameter settings. As $K$ increases from $5$ to $7$ or $S$ increases from $4$ to $16$, the computational cost increases significantly for MICP. We found that $S$ needs to be at least $8$ because otherwise the relaxation is too coarse, rendering the solution of MICP nearly useless as it will not satisfy the non-convex constraints even after a local NLP solve as post-processing. The cost of a typical MICP solve is in the order of tens of hours and in many cases hits the $24$ hours limit, especially when we use large $S$ or $T$. MINLP has relatively faster performance, and typically accomplishes the computation within $5-10$ hours. The cost of MICP increases superlinearly with both $S$ and $T$, while the cost of MINLP increases superlinearly with only $S$ but not $T$. This is because our number of binary variables in MINLP does not increase with $T$. On the other hand, MINLP cannot match the target curve well in 2 out of 10 test cases from \prettyref{table:results}, while MICP always achieves an ideal match. We further notice from \prettyref{fig:convergence} that MICP explores orders of magnitude more nodes than MINLP in the search tree of the BB algorithm, which is understandable due to a much larger number of binary variables. However, MINLP uses more time to explore each node due to a higher cost in solving a non-convex problem for each node. We observe that both optimizers update the solution less than $10$ times before convergence and most of the computations are devoted to detecting and pruning impossible cases. We conclude that MINLP achieves an overall better computational efficacy than MICP at a minor sacrifice of optimality.
\begin{figure*}[ht]
\centering
\begin{tabular}{ccc}
\includegraphics[width=0.32\textwidth]{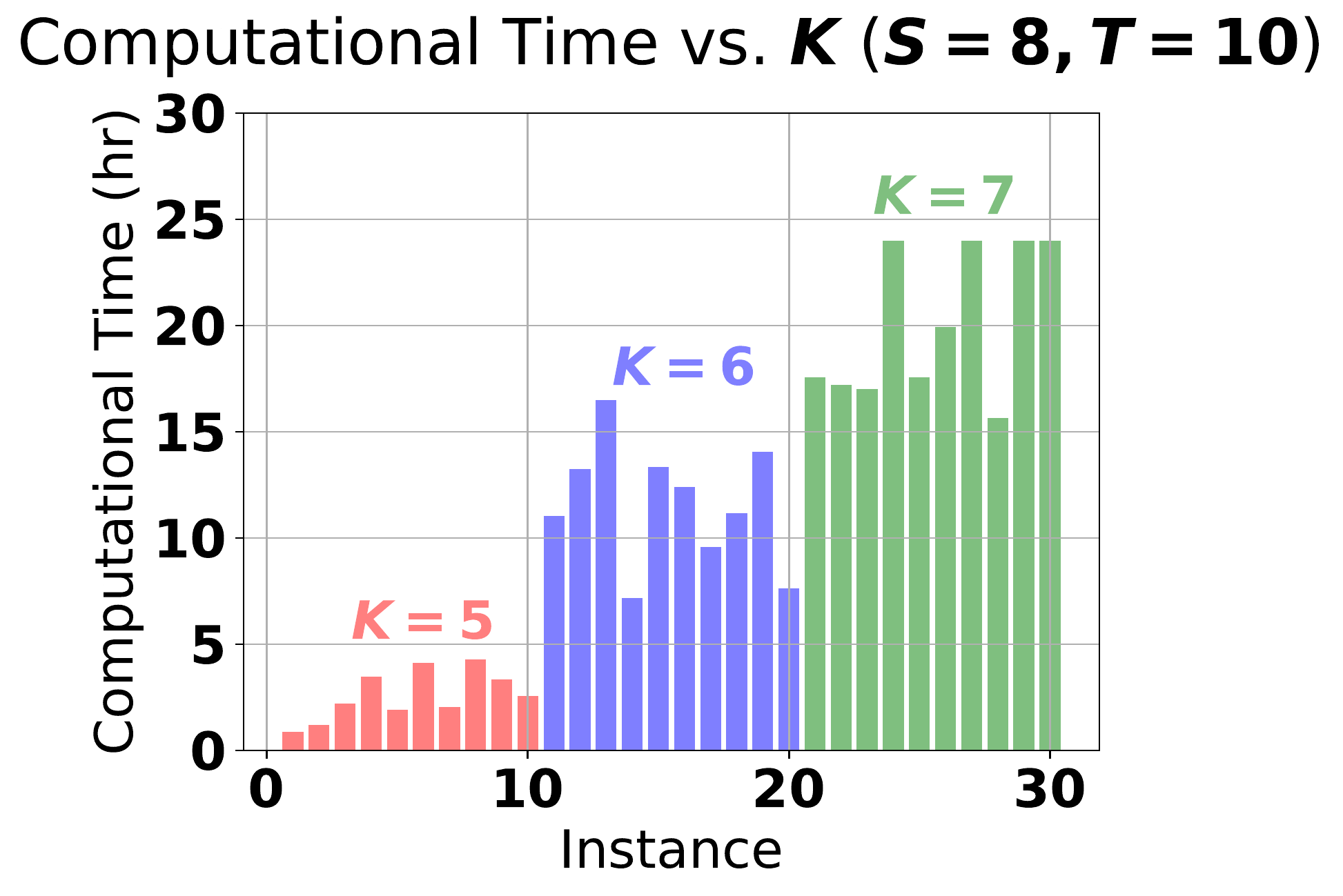}\put(-155,5){(a)}&
\includegraphics[width=0.32\textwidth]{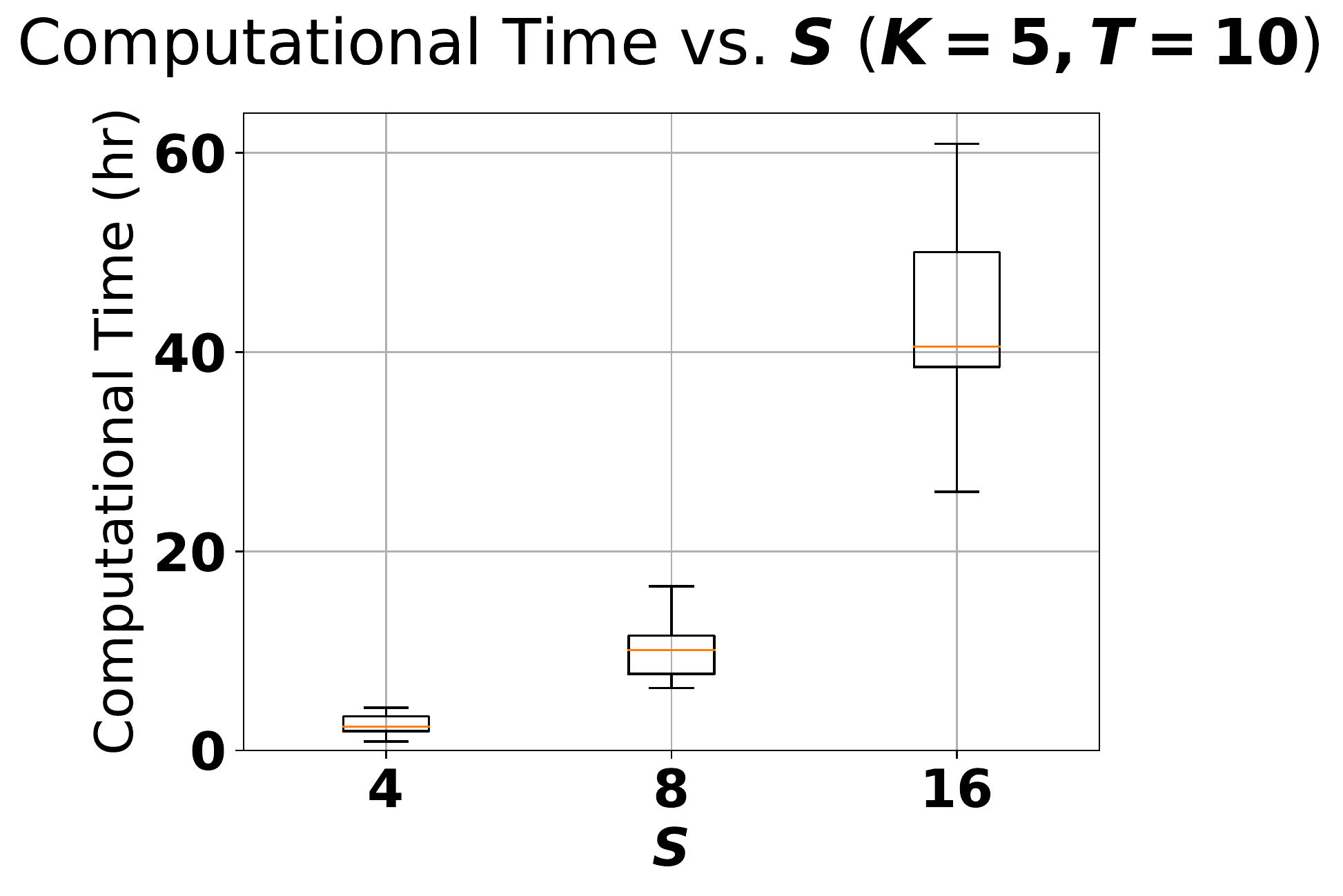}&
\includegraphics[width=0.32\textwidth]{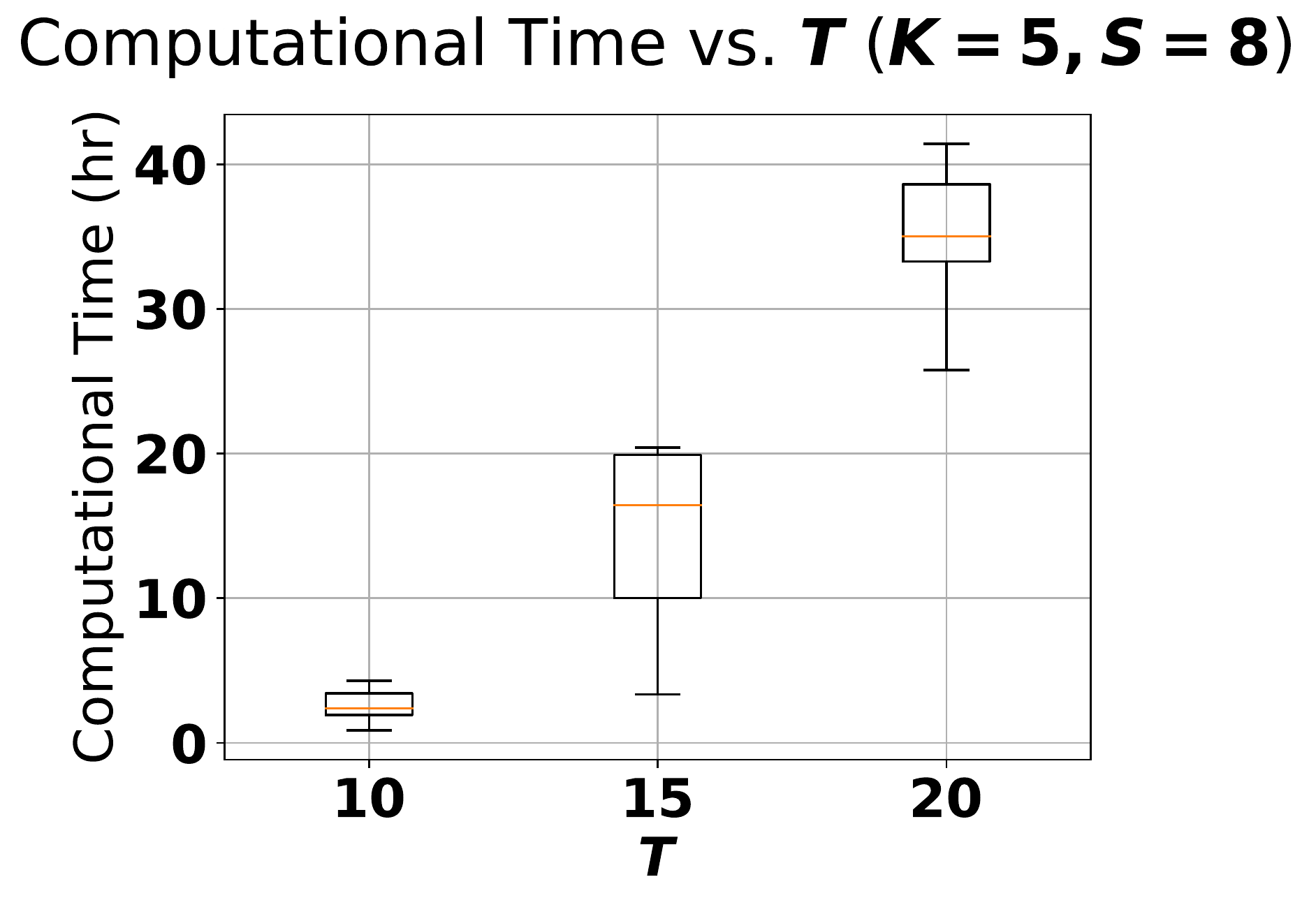}\\
\includegraphics[width=0.32\textwidth]{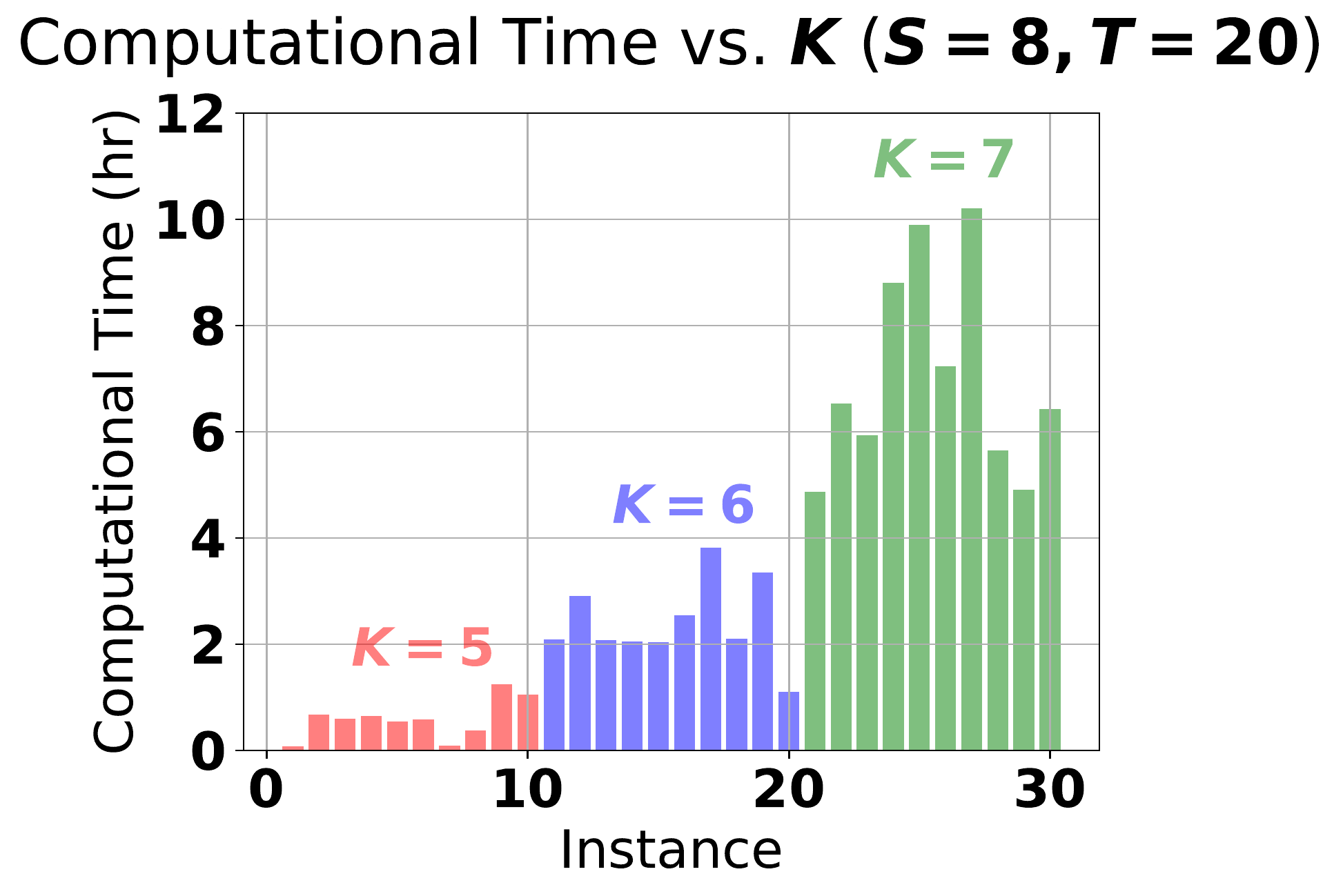}\put(-155,5){(b)}&
\includegraphics[width=0.32\textwidth]{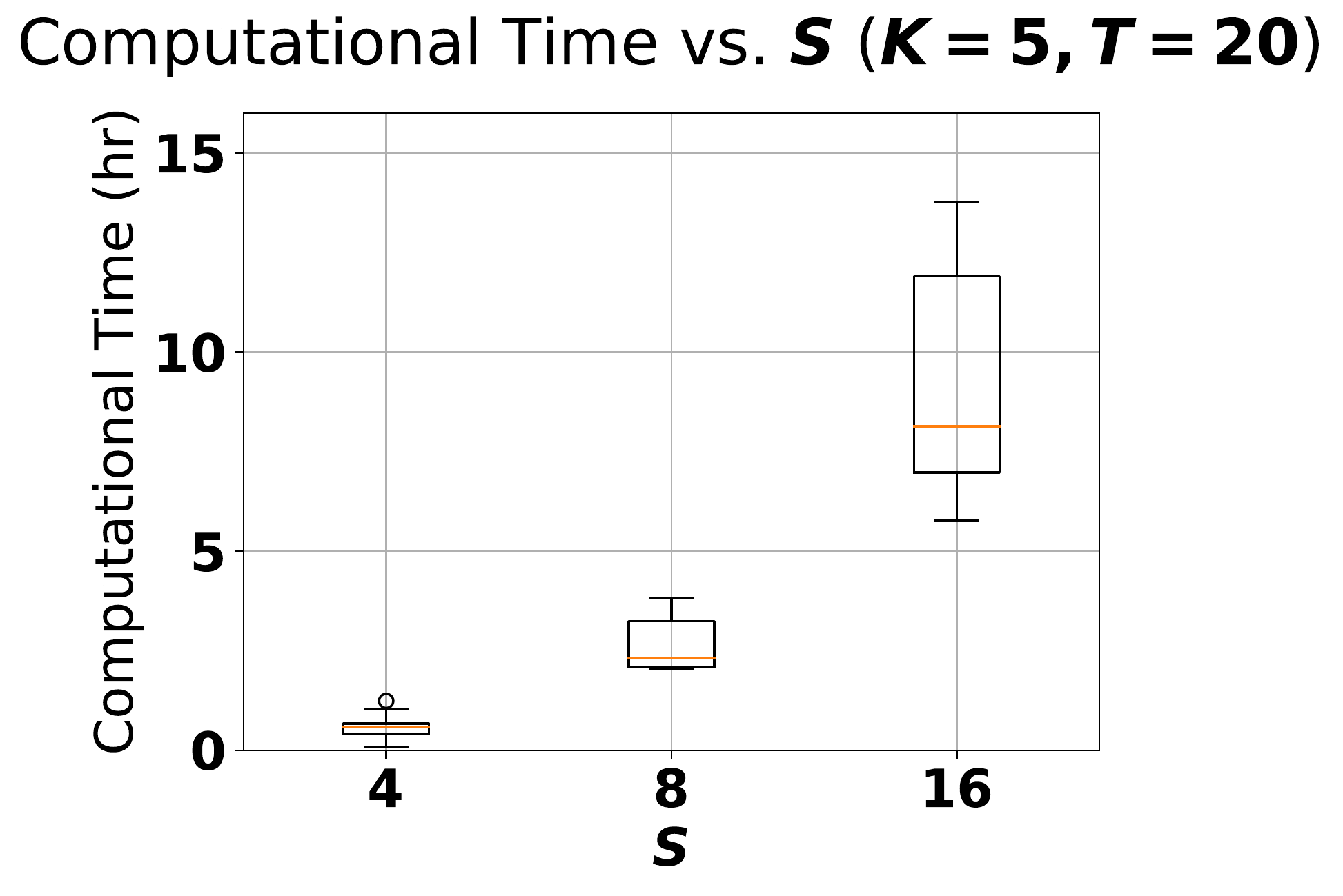}&
\includegraphics[width=0.32\textwidth]{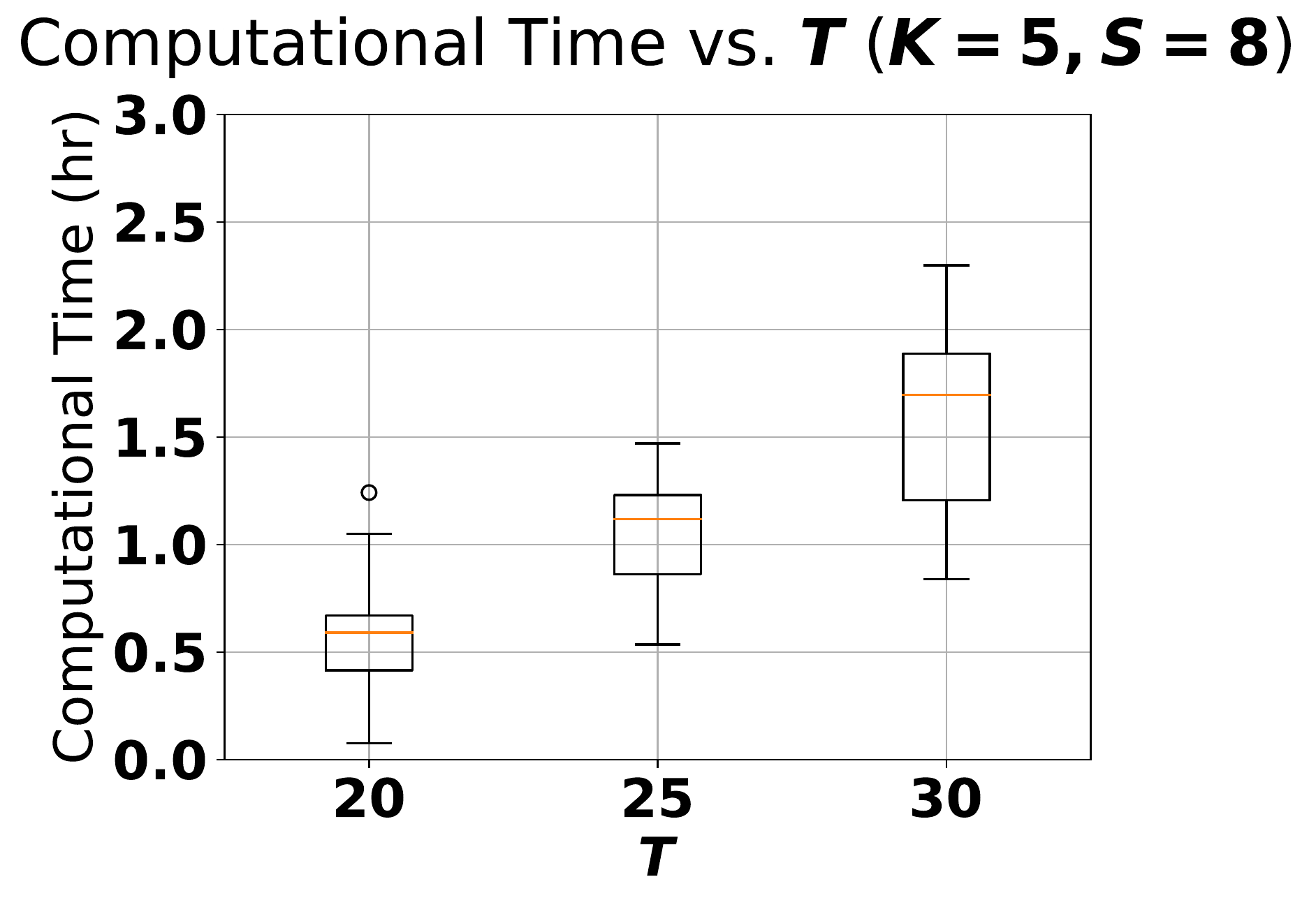}
\end{tabular}
\vspace{-5px}
\caption{\label{fig:costFunction} The cost of solving MICP (a) and MINLP (b) for 10 benchmark problems as a function of $K$, $S$, and $T$.}
\vspace{-5px}
\end{figure*}
\begin{figure}[ht]
\centering
\resizebox{.49\textwidth}{!}{\begin{tabular}{cc}
\includegraphics[height=0.24\textwidth]{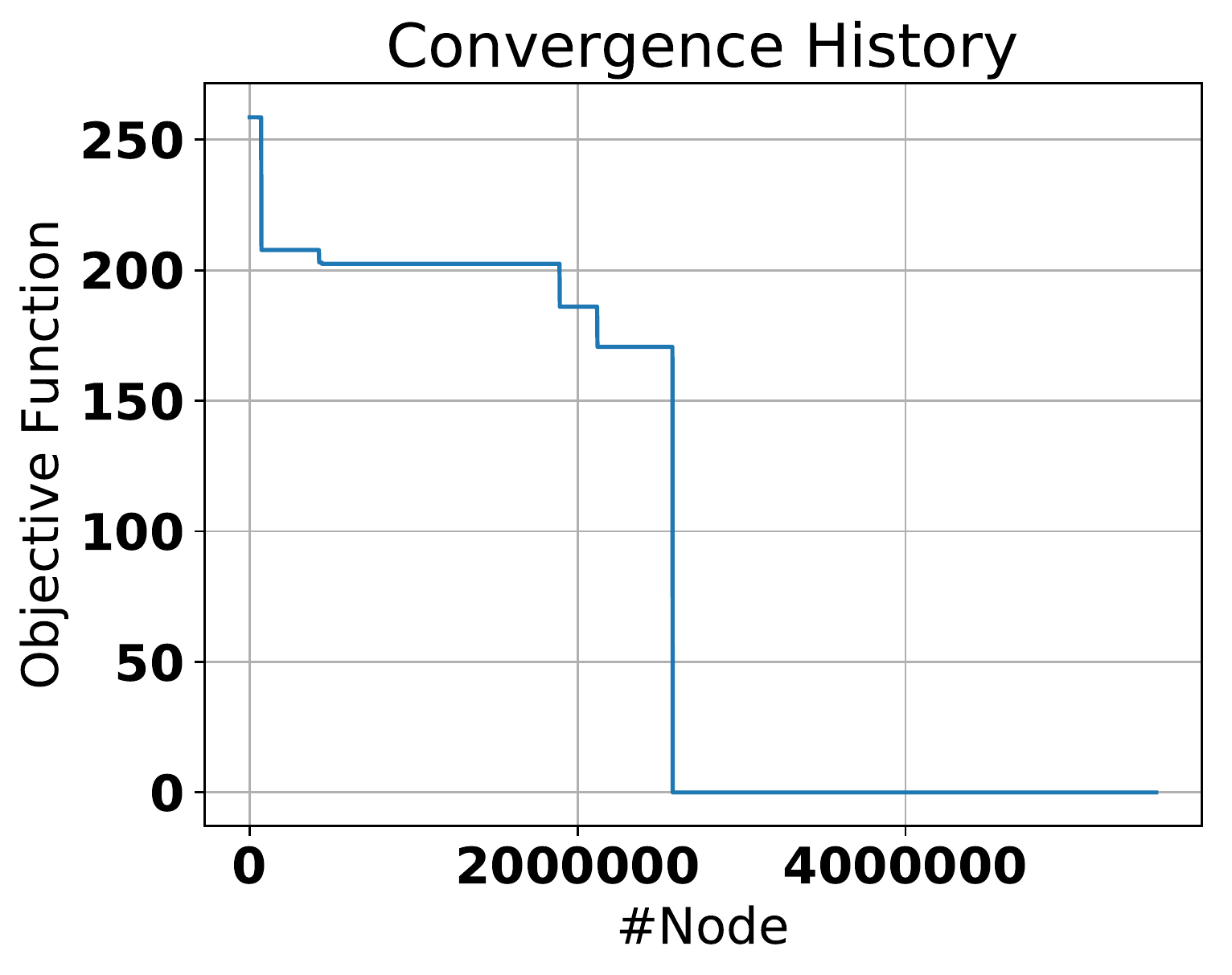}\put(-145,5){(a)}&
\includegraphics[height=0.24\textwidth]{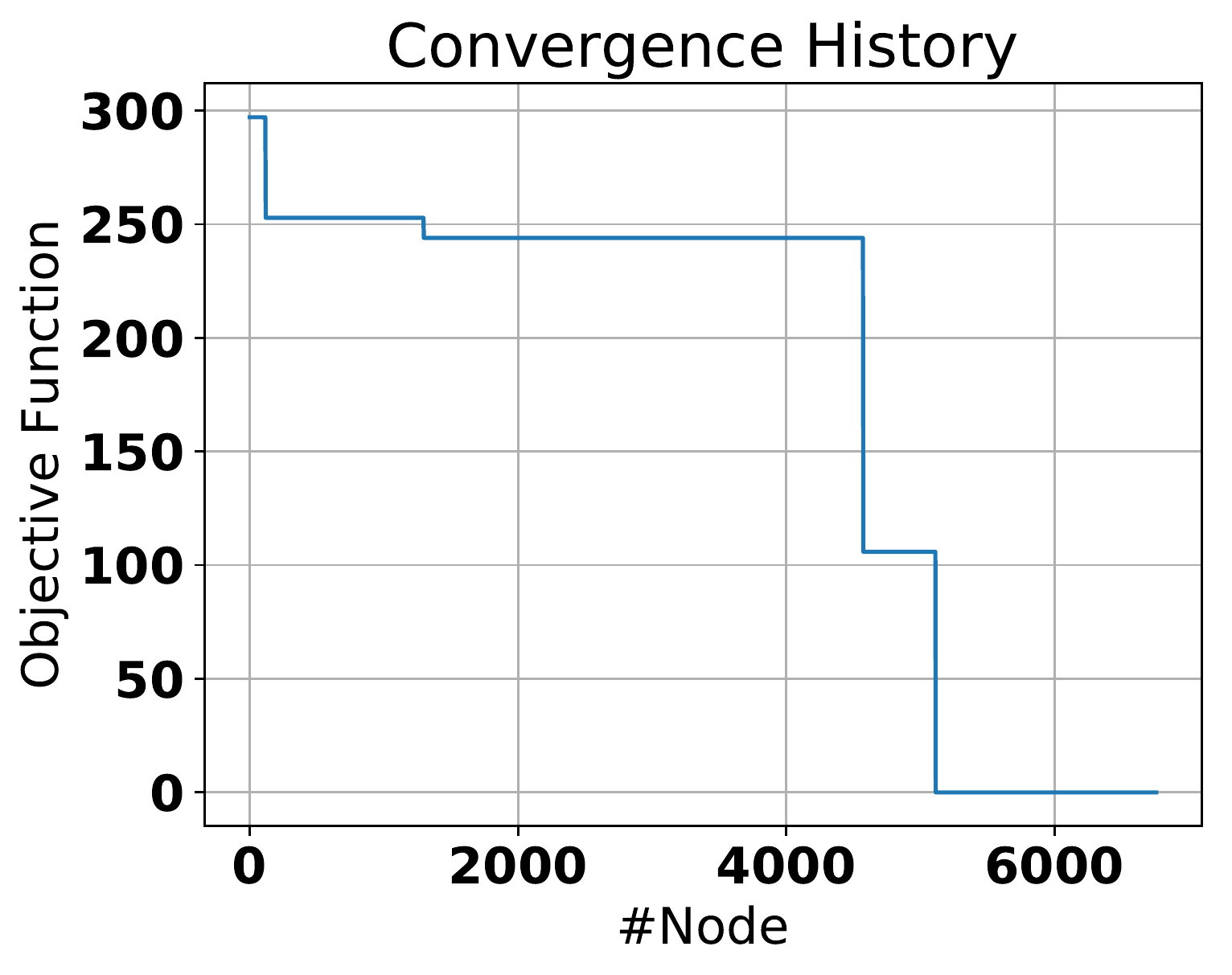}\put(-140,5){(b)}
\end{tabular}}
\caption{\label{fig:convergence} The convergence history of MICP (a) and MINLP (b) for the first problem in \prettyref{table:results}, plotted against the number of nodes explored on the BB search tree.}
\vspace{-5px}
\end{figure}

\subsection{Alternative User Interfaces}
\begin{figure*}[ht]
\centering
\begin{tabular}{cc}
\includegraphics[width=0.5\textwidth]{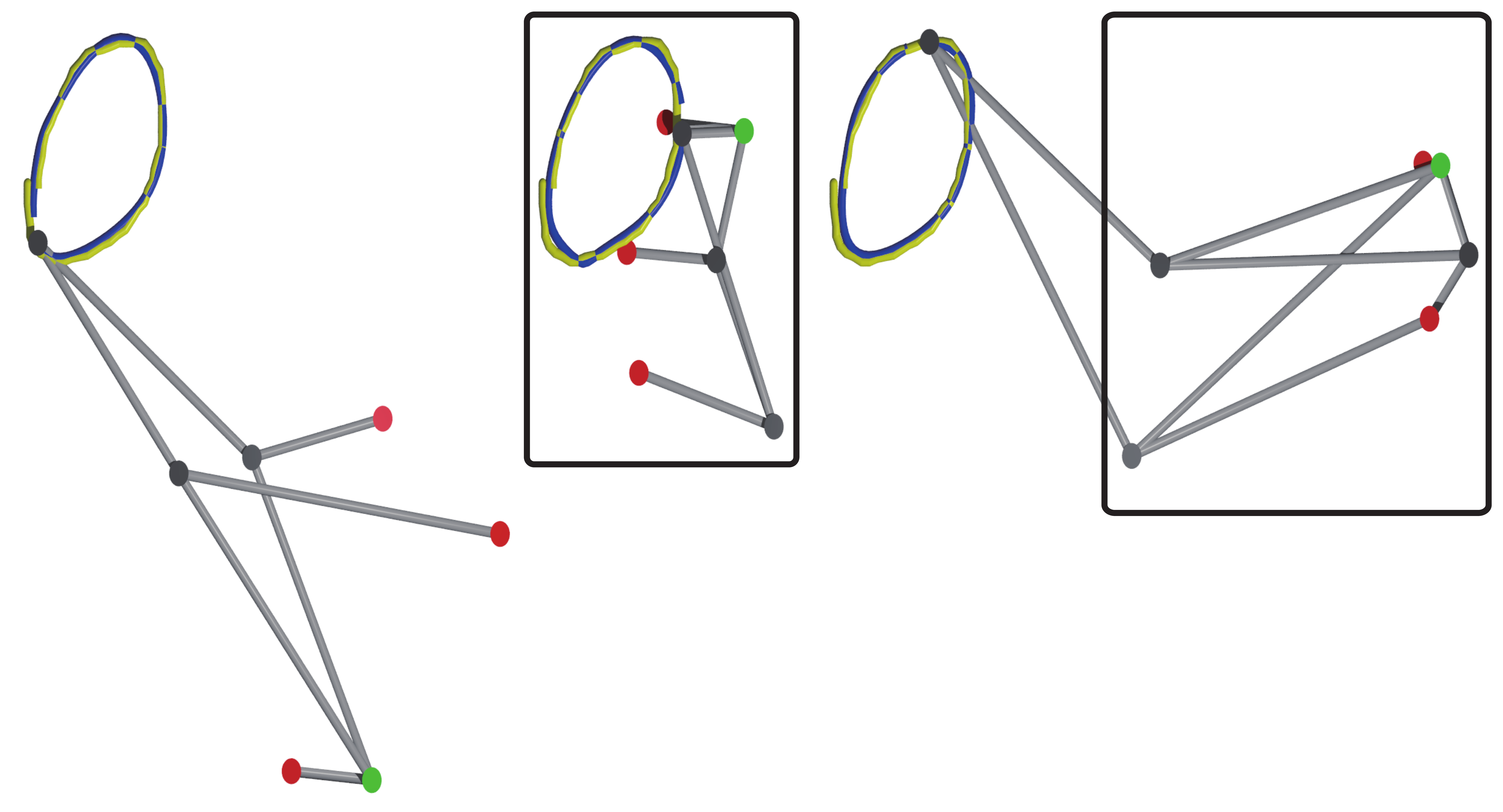}
\put(-238,105){(a)}
\put(-153,105){(b)}
\put(-105,105){(c)} &
\includegraphics[width=0.43\textwidth]{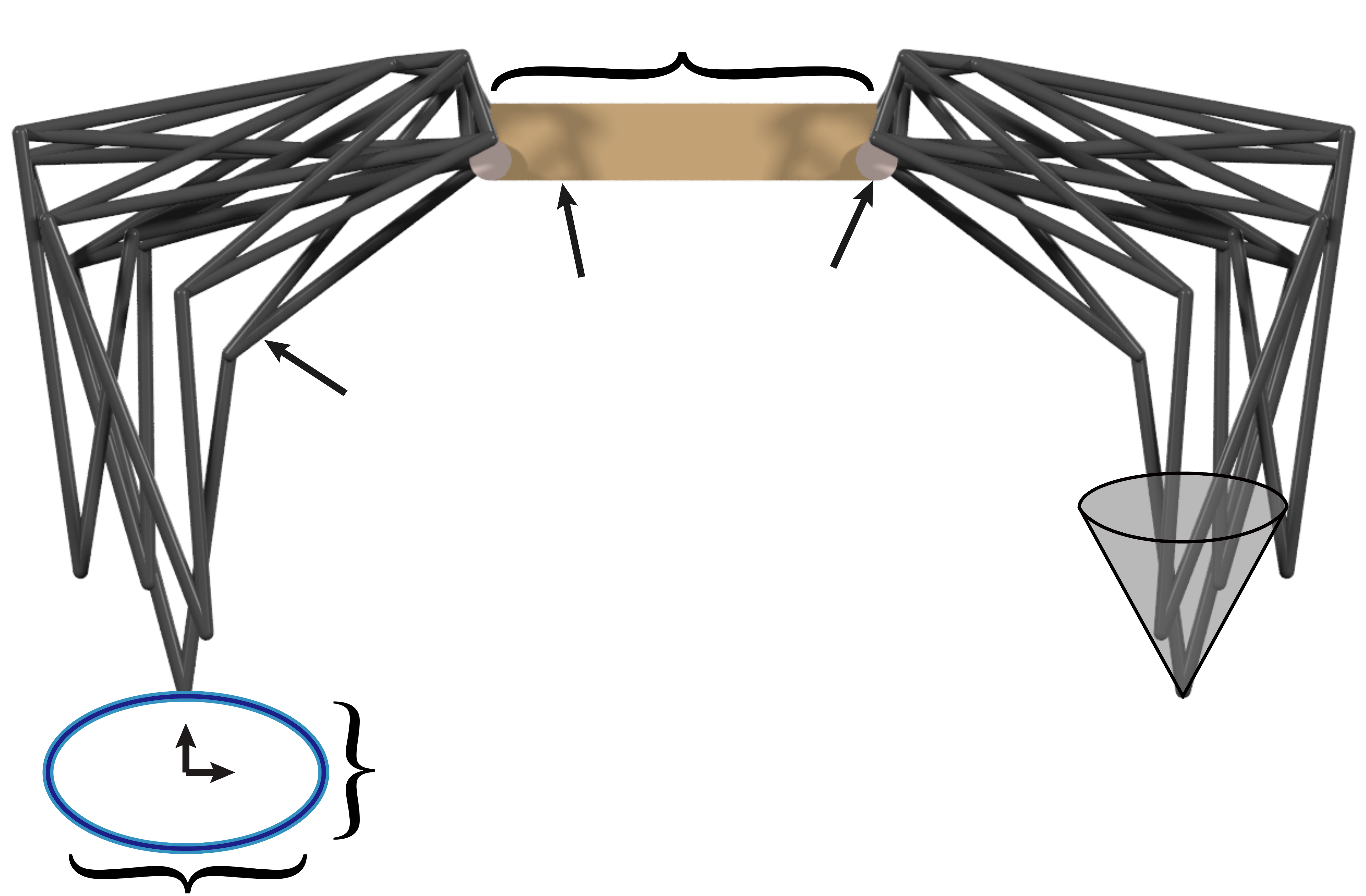}
\put(-178,15){\tiny$\Delta x$}
\put(-195,24){\tiny$\Delta y$}
\put(-185,-3){\tiny$s_x$}
\put(-154,18){\tiny$s_y$}
\put(-158,75){\tiny$\rho_l$}
\put(-125,92){\tiny$\rho_r$}
\put(-85,94){\tiny$\tau$}
\put(-90,87){\tiny{Motor Speed}}
\put(-109,135){\tiny$d_l$}
\put(-43,40){\tiny$\mu$}
\put(-90,30){(d)}
\end{tabular}
\caption{\label{fig:boxConstraint} We illustrate the effect of additional constraints on node positions. (a): no constraint; (b): box constraint that limits the entire structure in a small vicinity around the target curve; (c): box constraint that ensures that all the nodes (except for the end-effector) are a certain distance away from the target curve. (d): We mount the linkage structure (c) on a walking robot and optimize its walking distance with respect to the parameters: target curve translation $\Delta x,\Delta y(m)$, target curve scale $s_x,s_y(m)$, frictional coefficient $\mu$, robot mass density $\rho_r(kg/m^3)$, leg mass density $\rho_l(kg/m^3)$, motor torque $\tau(kgm^2/s^2)$, motor speed$(m/s)$, and body length $d_l(m)$.}
%\vspace{-5px}
\end{figure*}
\begin{figure}[ht]
\centering
\vspace{-5px}
\includegraphics[width=0.45\textwidth]{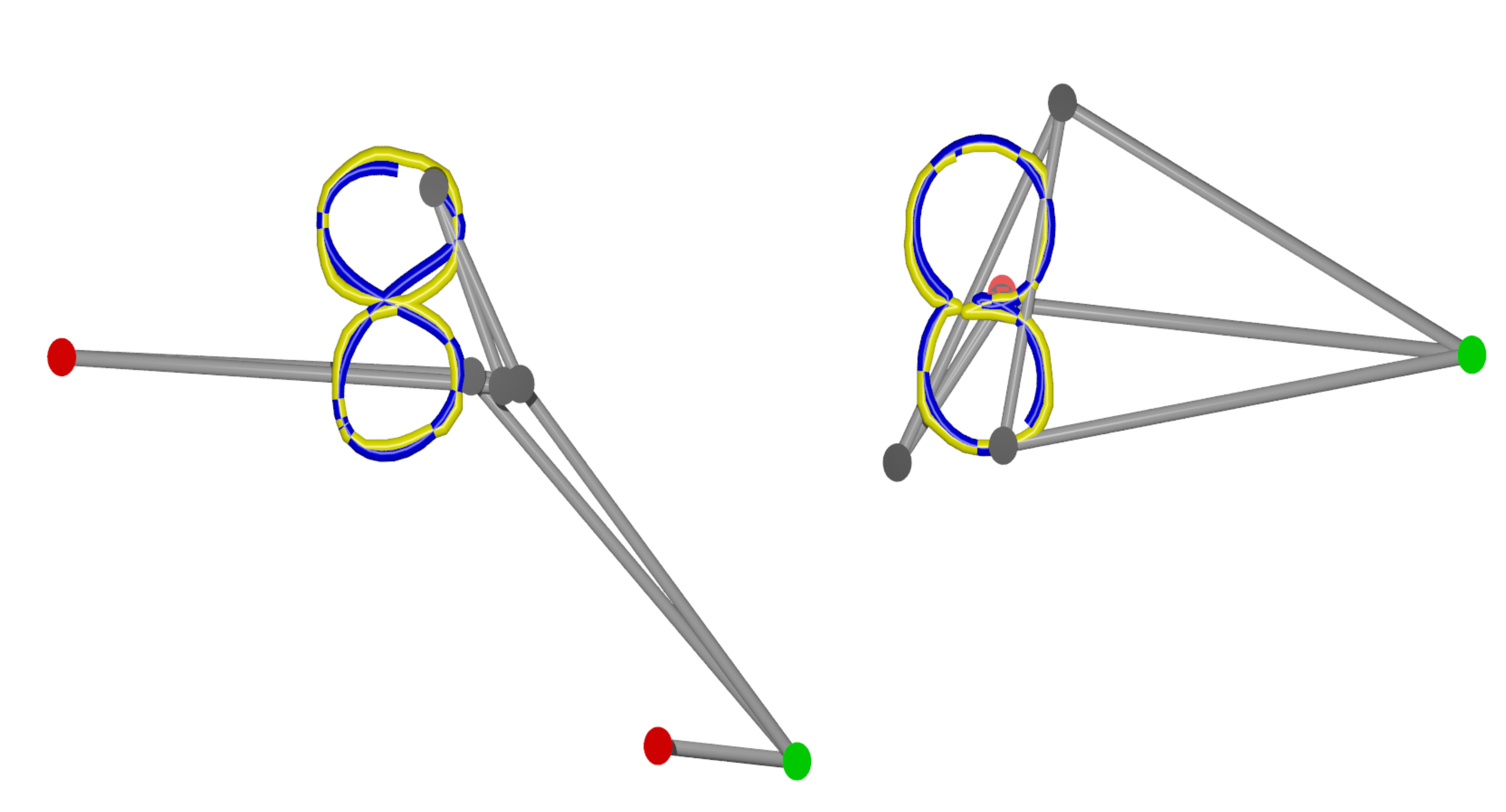}
\put(-185,70){(a)}
\put(-100,70){(b)}
\caption{\label{fig:orderIndependent} We illustrate the effect of allowing an arbitrary order to visit the set of points on the target 8-shaped curve. (a): the curve is genus-2 using fixed order; (b): the curve is genus-1 using arbitrary order.}
\vspace{-5px}
\end{figure}
\begin{wrapfigure}{r}{0.25\textwidth}
\centering
\includegraphics[width=0.24\textwidth]{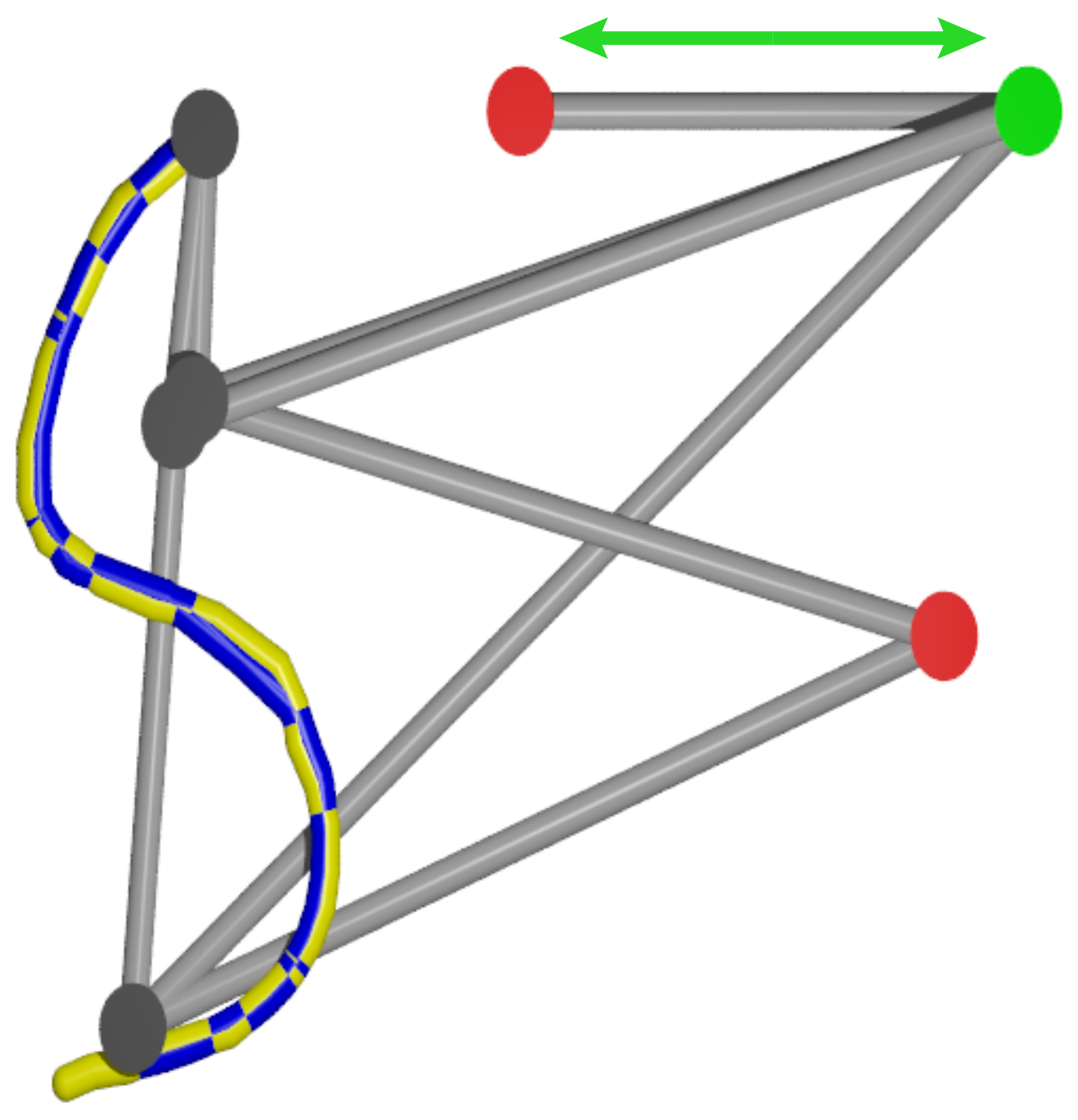}
\caption{\label{fig:linearMotion} \small{We illustrate a linkage structure that transforms linear motions into an S-shaped curve. The green node moves back and forth linearly along the green arrow.}}
\end{wrapfigure}
Our default user interface is for a designer to provide a target end-effector curve. However, other types of objective functions and hard constraints are possible. Our MIQCQP can take any non-convex objective function and constraints, while the MICP solver can take only convex objectives and constraints. In this section, we evaluate several additional user editing operations. We allow users to draw a box and constrain all the nodes (except for the end-effector node) to reside only in the given box, which can be formulated as four additional convex constraints. This is useful for a linkage structure to be mounted on a legged robot, where the non-end-effector nodes should be a certain distance away from the ground to avoid collisions in case of uneven terrains. Some mechanical toys have limited space in the gearbox and such constraints can be employed to fit the structure inside. In \prettyref{fig:boxConstraint}, we illustrate these two cases for the end-effector to trace out the same elliptical curve.

In addition, we allow users to draw a curve and optimize a linkage structure whose end-effector passes through the sampled points on the target curve with an arbitrary order, which is a typical case of coverage planning. For example, the order is unimportant for a planar linkage to hold a pen and fill out an area on a piece of paper. This requirement can be achieved by removing the first two equations in \prettyref{eq:orderDiscrete}, leaving only: $\E{d}_{11}^d(t^q)=\NN_1^d(t^q)-\TWO{X_C}{Y_C}$. We highlight such an example in \prettyref{fig:orderIndependent}, where the user provides an 8-shaped target curve. By default, the end-effector traces out a genus-2 curve, but it can also trace out a genus-1 curve to visit all the sampled points on the curve when the order is arbitrary. Finally, our formulation is not limited to rotational motors. In \prettyref{fig:linearMotion}, we illustrate a case with a linear motor, where the motor is moving according to: $\NN_1(t)=\TWO{X_C}{Y_C}+t\TWO{X_C'}{Y_C'}$. Such motion can be realized by replacing \prettyref{eq:orderDiscrete} with: $\NN_1(t^q)=\TWO{X_C}{Y_C}+t^q\TWO{X_C'}{Y_C'}$, where $X_C,Y_C,X_C',Y_C'$ are additional decision variables for the starting position and the moving direction. We speculate that several other motor types can be also realized by using similar techniques. Incorporating various motor types allows our formulation to be used in the modular design of mechanical systems, where the motion of the motor node is realized by another module.

\begin{figure*}[ht]
\centering
\includegraphics[width=0.9\textwidth]{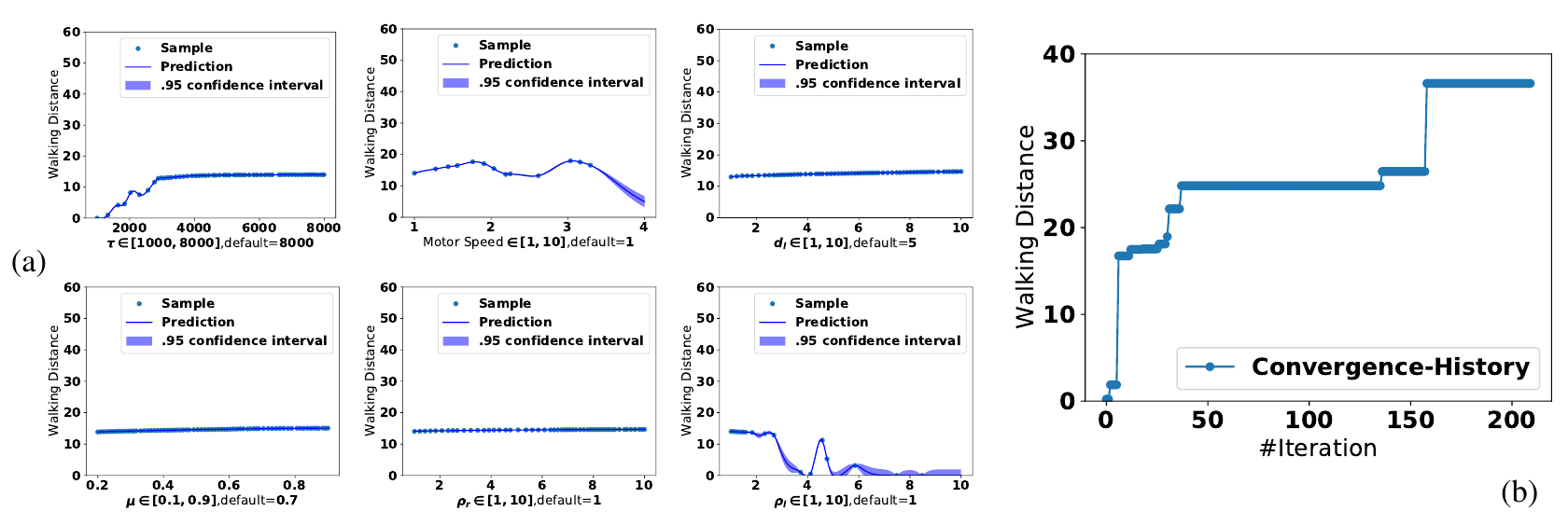}
\caption{\label{fig:Bayesian} (a): We optimize the robot's walking distance$(m)$ with respect to six parameters separately and we plot the objective landscape approximated by the Gaussian process. (b): We then optimize the three most important parameters: $\tau$, motor speed, and $\rho_l$ and plot the convergence history.}
\vspace{-5px}
\end{figure*}
\subsection{Robot Walking}
Prior works \citep{8793802,spielberg2017functional} have demonstrated that it is possible to design walking robots where linkage structures are used to transform rotational motion into loops of footsteps. Their end-effectors trace out an oval-shaped target curve, of which a well-known design is shown in \prettyref{fig:strandbeest}, \prettyref{fig:linkage}, and analyzed in \cite{nansai2013dynamic}. Although the above-mentioned work uses a manually designed linkage topology and geometry, they rely on an additional fine-tuning optimization to adjust the linkage's mounting points on the robot and geometric parameters. They show that such fine-tuning is essential to maximize the robot's performance, such as its walking speed.

\begin{figure}[ht]
\centering
\setlength{\tabcolsep}{0pt}
\begin{tabular}{c}
\resizebox{.45\textwidth}{!}{\begin{tabular}{cc}
\includegraphics[height=0.45\textwidth,trim=16cm 0cm 44cm 0cm,clip]
{tau.spd.dl.pickle.pdf}
\put(-210,120){$\rho_l$}
\put(-50 ,30){$\tau$}
\put(-195,30){Motor Speed}
\put(-95 ,210){(a): Walking Distance} &
\includegraphics[height=0.45\textwidth,trim=168cm 0cm 0cm 0cm,clip]
{tau.spd.dl.pickle.pdf}
\end{tabular}}\\
\includegraphics[width=0.45\textwidth]{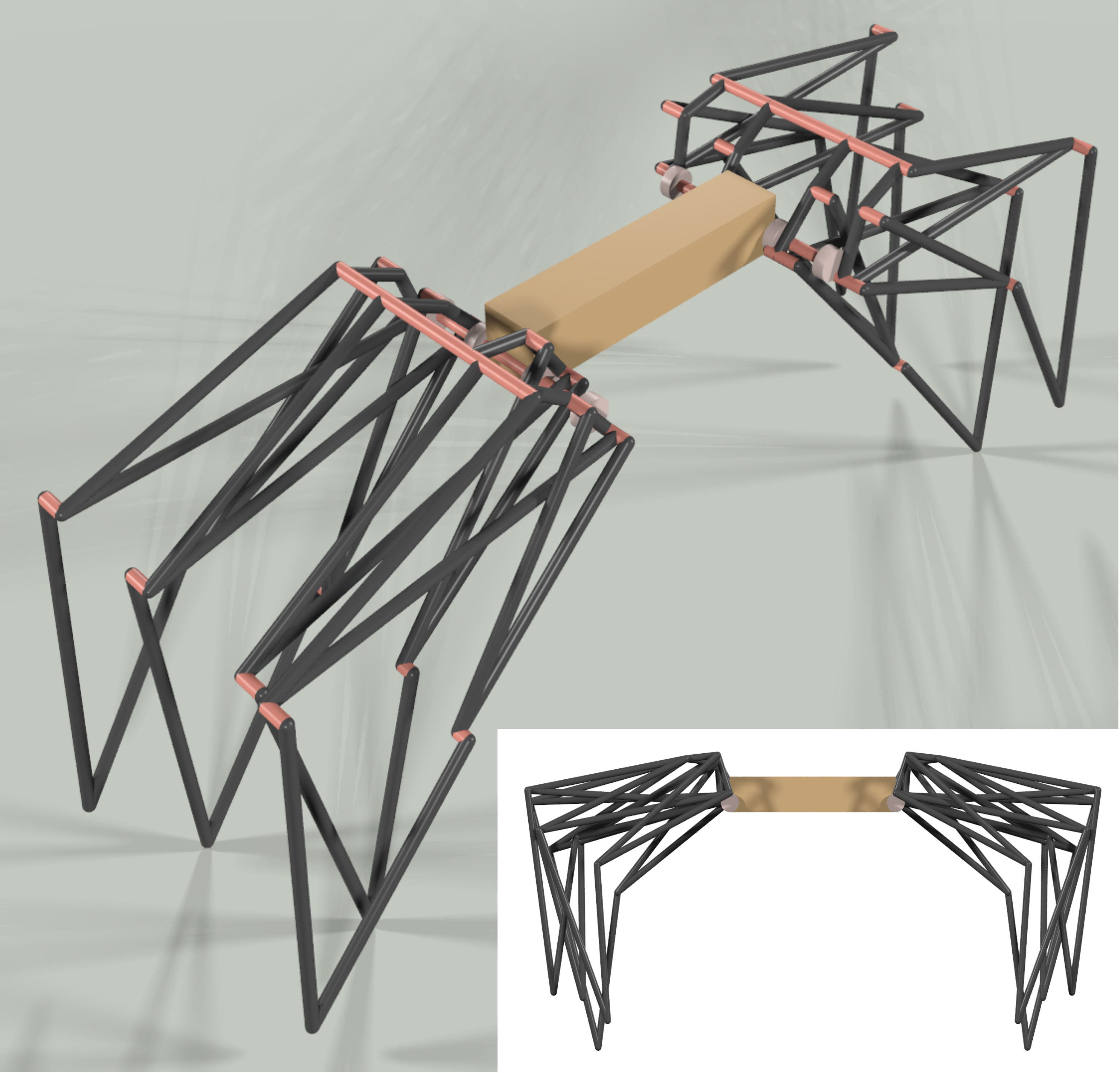}
\put(-65 ,10){(b)}
\end{tabular}
\caption{\label{fig:landscape} (a): After $200$ iterations of Bayesian optimization, we plot the approximate landscape of walking distance as a function of the three parameters: $\tau$, motor speed, and $\rho_l$ using volume rendering (red means higher objective function values). We also show the range and default value of each parameter. When a parameter is not optimized, we use its default value. (b): The best robot design in perspective and orthogonal view.}
\vspace{-5px}
\end{figure}
\refined{We explore the potential application of our optimized planar linkages in legged walking robotics, following a two-step semi-optimization approach similar to prior works \citep{soong2007simultaneous,erkaya2009determining,coros2013computational,Thomaszewski:2014:CDL:2601097.2601143}. Specifically, we first optimize a variety of different linkage geometries and topologies. We then manually choose one of these linkages as robot legs. Finally, we optimize the robot performance in an end-to-end manner.} To this end, we establish a testbed as illustrated in \prettyref{fig:boxConstraint} (d) where we mount a set of eight linkage structures shown in \prettyref{fig:boxConstraint} (c) onto a robot with a rectangular torso and two rotary motors, where four linkages are used as front legs and the other four as back legs. We simulate the robot walking on a flat terrain using the Bullet Physics Engine \citep{coumans2013bullet}. \refined{The robot motion is generated by creating a single rotary actuator on which a constant torque is applied, i.e. we assume the robot is not equipped with any sensor or controller. This is the case with many low-cost robots and mechanical toys.} There are several additional parameters to set up the robot simulator: the separation distance between the front and back legs $d_l$, the robot-to-ground frictional coefficient $\mu$, the motor torque $\tau$, the motor speed, the robot's mass density $\rho_r$, and the leg's mass density $\rho_l$. Since our formulation only considers the end-effector's curve and does not care about the robot's performance, we speculate that some fine-tuning is needed. We perform the fine-tuning by using Bayesian optimization \citep{eggensperger2013towards}, where our objective function is the distance traveled by the robot's center of mass over a simulated period of $10$ seconds. We first investigate which parameters must be fine tuned, so we run six passes of fine-tuning for each parameter; the results are summarized in \prettyref{fig:Bayesian} (a), where we observe the most significant performance increase by tuning $\tau$, the motor speed, and $\rho_l$. Next, we run another pass of fine-tuning jointly in these three parameters and observe a $4.3\times$ overall performance boost as shown in \prettyref{fig:Bayesian} (b) \refined{(with optimal values $\tau=7247(kgm^2/s^2)$, motor speed$=3.22(m/s)$, and $\rho_l=4.8(kg/m^3)$)}. Finally, we plot the landscape of the objective function that is approximated by using the Gaussian process in \prettyref{fig:landscape}, which is the output of Bayesian optimization after $200$ iterations. We can see that high objective function values only occupy a small fraction of the domain, so we conclude that fine-tuning the robot-mounting parameters is a necessity for linkage structures to gain high performance on robots. \refined{For example, we found that the minimal torque to drive the robot is $4500(kgm^2/s^2)$ as shown in \prettyref{fig:tau}.}
\begin{figure}[ht]
\centering
\includegraphics[width=0.45\textwidth]{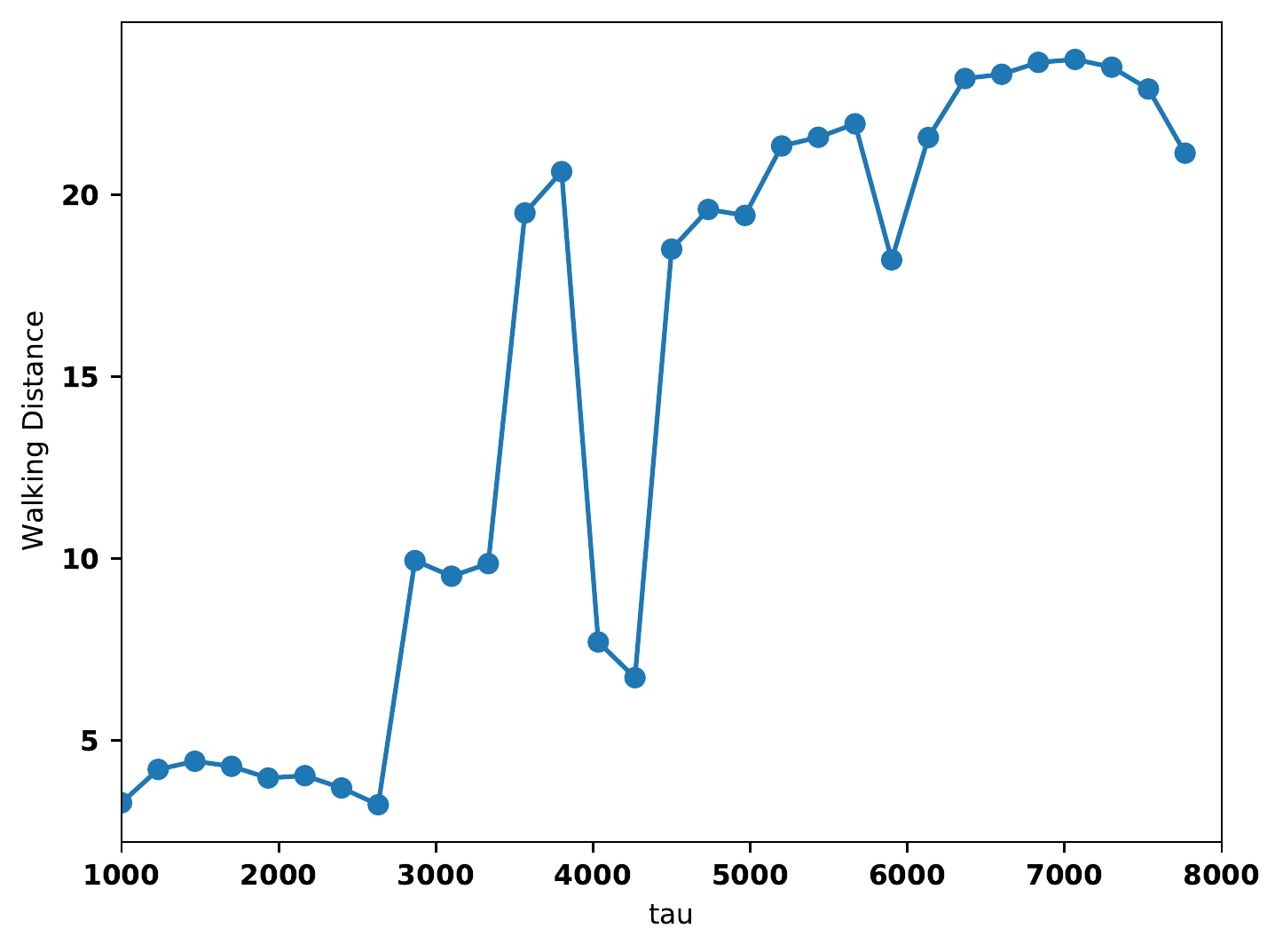}
\caption{\label{fig:tau} \refined{We fix the two variables at their optimal values (motor speed$=3.22(m/s)$, and $\rho_l=4.8(kg/m^3)$) and plot the robot's performance against the motor torque. As the torque changes between $[4500,7500](kgm^2/s^2)$, the performance only changes by $20\%$ of the optimal value.}}
\vspace{-5px}
\end{figure}

\begin{figure}[ht]
\centering
\setlength{\tabcolsep}{0pt}
\begin{tabular}{c}
\resizebox{.45\textwidth}{!}{\begin{tabular}{cc}
\includegraphics[height=0.45\textwidth,trim=0cm 0cm 40cm 0cm,clip]
{dx.sx.sy.pickle.pdf}
\put(-245,80){$s_x$}
\put(-65 ,60){$\Delta x$}
\put(-215,30){$s_y$}
\put(-95 ,210){(a): Walking Distance} &
\includegraphics[height=0.45\textwidth,trim=168cm 0cm 0cm 0cm,clip]
{dx.sx.sy.pickle.pdf}
\end{tabular}}\\
\includegraphics[width=0.45\textwidth]{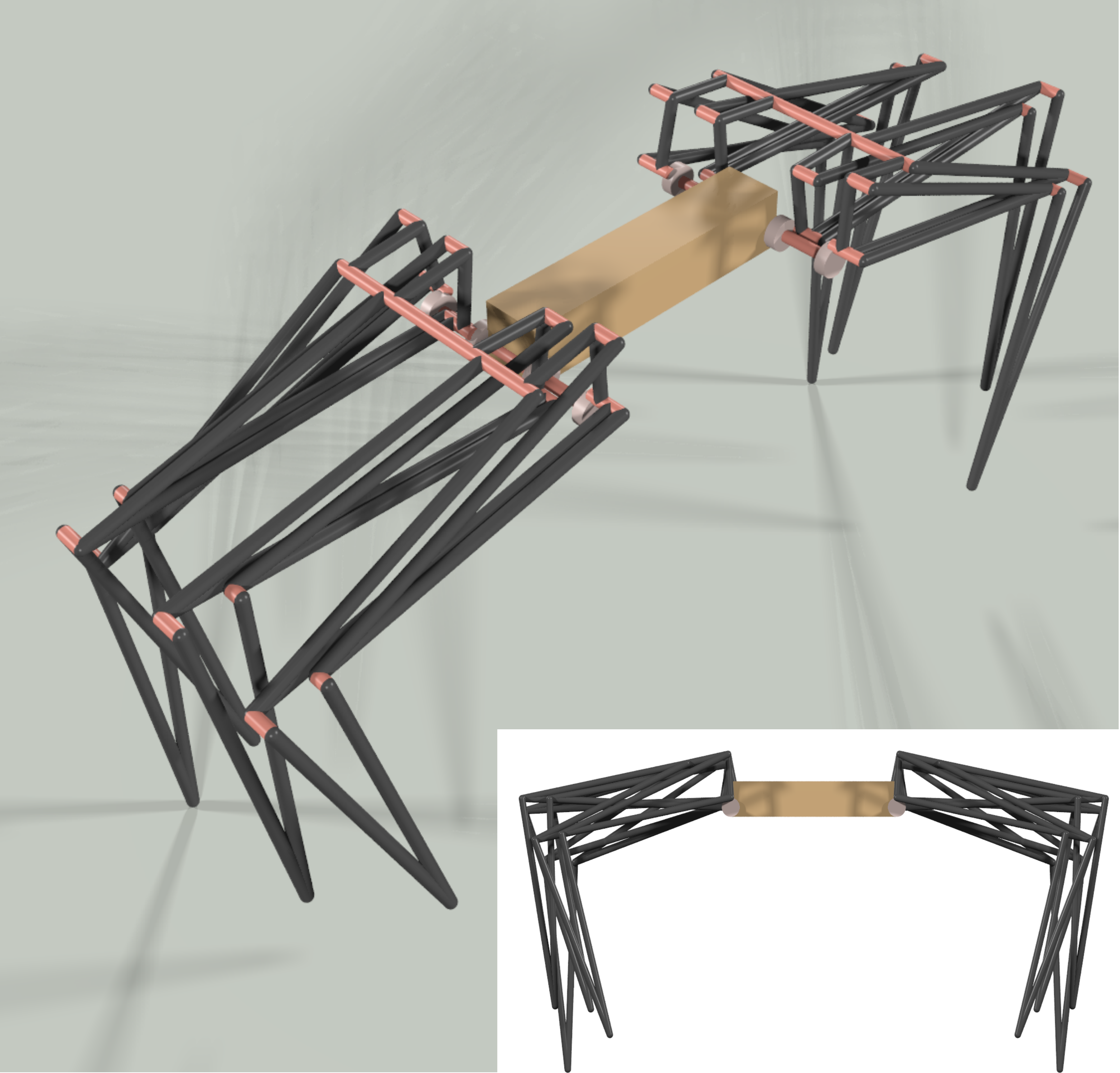}
\put(-65 ,10){(b)}
\end{tabular}
\caption{\label{fig:landscapeGeometry} (a): After $200$ iterations of Bayesian optimization, we plot the approximate landscape of walking distance as a function of the three parameters: $\Delta x$, $s_x$, and $s_y$ using volume rendering (red means higher objective function values). (b): The best robot design in perspective and orthogonal view.}
\vspace{-5px}
\end{figure}
The above fine-tuning modifies the mounting of linkage structure on the robot, while the geometry of the structure is fixed. We perform a separate fine-tuning that focuses on the geometry. There are too many parameters that specify the geometry and Bayesian optimization does not scale well to such high-dimensional decision spaces. To tackle this issue, we only modify the target curve and use local optimization to change the shape of the linkage structure. Specifically, we introduce four parameters: $\TWO{\Delta x}{\Delta y}\in[-3,3]^2$ and $\TWO{s_x}{s_y}\in[1,3]^2$ specify the translation and scaling of the target curve, respectively. Given $\TWO{\Delta x}{\Delta y}$ and $\TWO{s_x}{s_y}$, we use the L-BFGS-B algorithm to minimize the following local objective function where gradients are calculated using \prettyref{alg:forward}:
\begin{align*}
\int_0^{2\pi}
\left\|\NN_N(t)-\MTT{s_x}{}{}{s_y}\NN_N^*(t)-\TWOC{\Delta x}{\Delta y}\right\|^2 dt,
\end{align*}
where $\NN_N^*(t)$ is the target curve of \prettyref{fig:boxConstraint}. We fix the position of all fixed nodes, the center of motor $\TWO{X_C}{Y_C}$, and the radius of motor $r$, so that the linkage can be mounted on the robot in the same way. After local optimization, we mount the linkage structure on the robot and compute its walking distance over $10$ seconds of simulation. We optimize the walking distance with respect to the four parameters $\TWO{\Delta x}{\Delta y}$ and $\TWO{s_x}{s_y}$ by using Bayesian optimization. These four parameters are default to $\Delta x=\Delta y=0$ and $s_x=s_y=1$. We found that $\Delta x,s_x,s_y$ are the three most influential parameters. The walking distance as a function of these parameters and the optimized robot design are shown in \prettyref{fig:landscapeGeometry}. We observe that there are many designs leading to high walking distances. Therefore, the robot performance is not sensitive to the linkage geometry.

%% file: table.tex
\begin{figure*}
\centering
\includegraphics[width=0.88\textwidth]{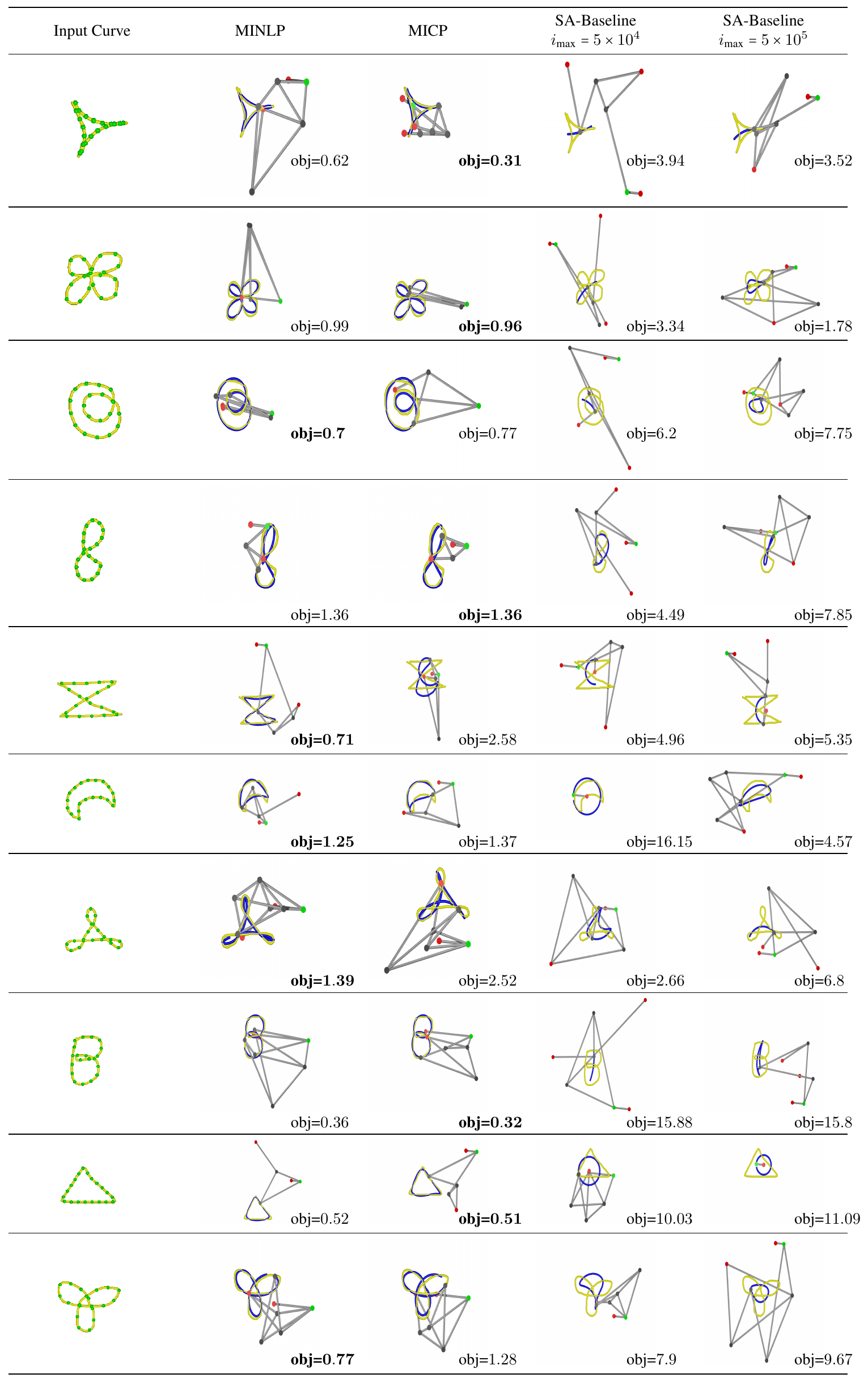}
\caption{\small{\label{table:results} We show 10 testing curves and the results optimized by our three algorithms and we evaluated two variants of SA-baseline with different iteration numbers. On the left, we show the user's input curves and $T=20$ sample points are drawn in yellow and green, respectively. For each optimized linkage structure, the fixed and movable nodes are shown in red and gray, respectively. The actual end-effector's curve is drawn in blue, and we mark the optimal objective function value on the lower-right corner.}}
%\vspace{-5px}
\end{figure*}

%% file: conclusion.tex
\section{Conclusions and Discussions}
We have proposed a deterministic algorithmic framework for optimizing a large subset of planar linkages, such that the end-effector traces out a curve that matches the user-specified target curve. We show that, by modeling the linkage structure by using maximal instead of minimal coordinates, the joint optimization of both the topology and geometry can be reformulated as an infinite-dimensional, non-convex MIQCQP. We further proposed a discretization scheme and three algorithms to solve MIQCQP approximately. Our first algorithm relaxes non-convex constraints as a disjoint convex set, allowing an MICP solver to find the global optima of the relaxed problem. Our second algorithm uses MINLP to find local feasible solutions via SQP. We highlight that, compared to the SA-baseline, our deterministic algorithm achieves $9.3\times$ higher optimality in a row of benchmarks (measured by the average ratio of objective function values in \prettyref{table:results}) and can take various additional constraints. These promising results can induce several avenues of future work.

Our work can find small linkage structures with up to $7$ nodes. Even at such a small scale, solving MIQCQP is still computationally intensive, taking tens of hours on a desktop machine. There are, however, several ways to further accelerate the algorithm. First, with the availability of multi-core processors, the BB algorithm can be parallelized by exploring multiple nodes simultaneously. Second, the efficacy of the BB algorithm is closely related to the heuristic rules for expanding the search tree, generating cutting planes, and warm-starting the node solutions, for which dedicated heuristic rules can be designed for our problem. Finally, the relaxation scheme of non-convex constraints can be optimized to reduce the approximation error as described in \cite{nagarajan2019adaptive}. For example, the sample positions in \prettyref{fig:upperBound} can be adaptively selected. This method can also minimize the number of samples ($S$) and reduce the number of binary variables.

\refined{Our current implementation does not allow users to specify the timing for the end-effector to reach each sample point on the target curve. Currently, we support two default timing schemes: 1) even sampling the target curve and assuming equal travel time between consecutive samples; 2) arbitrary travel time and order for all the samples, as illustrated in \prettyref{fig:orderIndependent}. This is due to two reasons. First, it is difficult and non-intuitive for users to specify the exact timing via a GUI interface. Second, at such a small scale with up to $7$ nodes, we have not observed significantly different designs using different timing schemes and cases such as \prettyref{fig:orderIndependent} are rare. We expect a larger solution space would lead to a variety of designs corresponding to more deliberate timing specification.}

Our method can only approximate the solutions of MIQCQP and sometimes can fail at finding a feasible solution. \refined{Indeed, we only impose non-singular constraints (\prettyref{eq:area}) at discrete time instances to avoid infinite constraints and \prettyref{lem:nonsingular} is not guaranteed to hold as a result. Our MICP solver further approximates the non-convex constraints as piecewise convex ones.} By comparison, the SA-baseline is guaranteed to return a solution, although it can drift far from the user's requirement. It is worthwhile to explore an approximation scheme for relaxing MIQCQP. One promising direction is to consider the semidefinite lower-bound that turns a quadratic constraint into a linear matrix inequality \citep{vandenberghe1996semidefinite}. More generally, the sum-of-squares programming allows any polynomial optimization to be converted into a semidefinite programming problem \citep{laurent2009sums}, and the conversion is exact under certain conditions. Such conversions can be used to derive the lower-bound in BB algorithms. Recent work \citep{9341017} has applied this idea to the inverse kinematic problems of sequential manipulators. The main advantage of sum-of-squares programming is that the lower bound can be made arbitrarily tight.

\refined{Finally, a major limitation of our method is that we only optimize the kinematics and geometric features of the linkage structure, which is not sufficient for many robotic applications, especially when deployed onto a physical robot platform. In our robot walking results for example, optimizing the dynamics properties, e.g., joint torques, frictional coefficients, mass distributions, is key to the overall final performance. Our current experiments optimize these dynamics properties using Bayesian exploration as a separate post-process, assuming fixed geometry and topology. This is a standard approach used by several prior works to automatically tune the dynamics properties. For example, \cite{bai2018reducing} proposed an optimization method to reduce the vibration. \cite{feng2002new} optimized the mass distribution to reduce the needed joint force. Truss optimization \citep{sokol201199} typically maximizes the strength of a linkage structure under external forces. Joint formulations such as \cite{soong2007simultaneous} have also been proposed that simultaneously minimize the motor torques and maximize the structure strength. \cite{erkaya2009determining} used simulated annealing to adjust multiple dynamics parameters under the influence of joint clearance. However, we expect that higher performance can be achieved by considering kinematics and dynamics into a single, joint optimization formulation. In many applications, the dynamic properties, e.g., material densities and motor torques, are pre-determined by hardware specifications, and the optimized linkage structures should satisfy these specifications as hard constraints that may also include collision handling~\cite{govindaraju2005quick,kim2002fast}. Unfortunately, formulating these considerations would significantly increase the complexity and computational time, so we leave them as future work.}